\documentclass[11pt]{article}
\pdfoutput=1

\usepackage[letterpaper, left=1in, right=1in, top=1in, bottom=1in]{geometry}

\usepackage[dvipsnames]{xcolor}
\usepackage[colorlinks=true, linkcolor=blue!70!black, citecolor=blue!70!black,urlcolor=blue!70!black]{hyperref}
\usepackage{microtype}

\usepackage{algorithm}

\usepackage{natbib}
\bibpunct{(}{)}{;}{a}{,}{,}

\usepackage{amsthm}
\usepackage{mathtools}
\usepackage{amsmath}
\usepackage{bbm}
\usepackage{amsfonts}
\usepackage{amssymb}

\usepackage{MnSymbol} 

\usepackage{xpatch}

\theoremstyle{definition}  

\newtheorem{lemma}{Lemma}
\newtheorem{assumption}{Assumption}

\newtheorem{proposition}{Proposition}

\theoremstyle{plain}

\newtheorem{theorem}{Theorem}
\newtheorem{definition}{Definition}

\xpatchcmd{\proof}{\itshape}{\normalfont\proofnameformat}{}{}
\newcommand{\proofnameformat}{\bfseries}

\usepackage{prettyref}
\newcommand{\pref}[1]{\prettyref{#1}}
\newcommand{\pfref}[1]{Proof of \prettyref{#1}}
\newcommand{\savehyperref}[2]{\texorpdfstring{\hyperref[#1]{#2}}{#2}}
\newrefformat{eq}{\savehyperref{#1}{\textup{(\ref*{#1})}}}
\newrefformat{eqn}{\savehyperref{#1}{Eq.~\textup{(\ref*{#1})}}}
\newrefformat{lem}{\savehyperref{#1}{Lemma~\ref*{#1}}}
\newrefformat{def}{\savehyperref{#1}{Definition~\ref*{#1}}}
\newrefformat{line}{\savehyperref{#1}{line~\ref*{#1}}}
\newrefformat{thm}{\savehyperref{#1}{Theorem~\ref*{#1}}}
\newrefformat{corr}{\savehyperref{#1}{Corollary~\ref*{#1}}}
\newrefformat{sec}{\savehyperref{#1}{Section~\ref*{#1}}}
\newrefformat{app}{\savehyperref{#1}{Appendix~\ref*{#1}}}
\newrefformat{ass}{\savehyperref{#1}{Assumption~\ref*{#1}}}
\newrefformat{ex}{\savehyperref{#1}{Example~\ref*{#1}}}
\newrefformat{fig}{\savehyperref{#1}{Figure~\ref*{#1}}}
\newrefformat{alg}{\savehyperref{#1}{Algorithm~\ref*{#1}}}
\newrefformat{rem}{\savehyperref{#1}{Remark~\ref*{#1}}}
\newrefformat{conj}{\savehyperref{#1}{Conjecture~\ref*{#1}}}
\newrefformat{prop}{\savehyperref{#1}{Proposition~\ref*{#1}}}
\newrefformat{proto}{\savehyperref{#1}{Protocol~\ref*{#1}}}
\newrefformat{prob}{\savehyperref{#1}{Problem~\ref*{#1}}}
\newrefformat{claim}{\savehyperref{#1}{Claim~\ref*{#1}}}

\DeclarePairedDelimiter{\abs}{\lvert}{\rvert} %
\DeclarePairedDelimiter{\brk}{[}{]}
\DeclarePairedDelimiter{\crl}{\{}{\}}
\DeclarePairedDelimiter{\prn}{(}{)}
\DeclarePairedDelimiter{\nrm}{\|}{\|}
\DeclarePairedDelimiter{\tri}{\langle}{\rangle}

\DeclareMathOperator{\En}{\mathbb{E}}

\DeclareMathOperator*{\argmin}{arg\,min} 

\newcommand{\wh}[1]{\widehat{#1}}

\def\ddefloop#1{\ifx\ddefloop#1\else\ddef{#1}\expandafter\ddefloop\fi}
\def\ddef#1{\expandafter\def\csname bb#1\endcsname{\ensuremath{\mathbb{#1}}}}
\ddefloop ABCDEFGHIJKLMNOPQRSTUVWXYZ\ddefloop
\def\ddefloop#1{\ifx\ddefloop#1\else\ddef{#1}\expandafter\ddefloop\fi}
\def\ddef#1{\expandafter\def\csname b#1\endcsname{\ensuremath{\mathbf{#1}}}}
\ddefloop ABCDEFGHIJKLMNOPQRSTUVWXYZ\ddefloop
\def\ddef#1{\expandafter\def\csname c#1\endcsname{\ensuremath{\mathcal{#1}}}}
\ddefloop ABCDEFGHIJKLMNOPQRSTUVWXYZ\ddefloop
\def\ddef#1{\expandafter\def\csname h#1\endcsname{\ensuremath{\widehat{#1}}}}
\ddefloop ABCDEFGHIJKLMNOPQRSTUVWXYZ\ddefloop
\def\ddef#1{\expandafter\def\csname hc#1\endcsname{\ensuremath{\widehat{\mathcal{#1}}}}}
\ddefloop ABCDEFGHIJKLMNOPQRSTUVWXYZ\ddefloop
\def\ddef#1{\expandafter\def\csname t#1\endcsname{\ensuremath{\widetilde{#1}}}}
\ddefloop ABCDEFGHIJKLMNOPQRSTUVWXYZ\ddefloop
\def\ddef#1{\expandafter\def\csname tc#1\endcsname{\ensuremath{\widetilde{\mathcal{#1}}}}}
\ddefloop ABCDEFGHIJKLMNOPQRSTUVWXYZ\ddefloop

\newcommand{\ls}{\ell}
\newcommand{\ind}{\mathbbm{1}}    

\newcommand{\eps}{\epsilon}
\newcommand{\veps}{\varepsilon}

\newcommand{\ldef}{\vcentcolon=}
\newcommand{\rdef}{=\vcentcolon}

\newcommand{\mb}[1]{\boldsymbol{#1}}

\newcommand{\trn}{\intercal}

\renewcommand{\trn}{\dagger}

\usepackage{algorithm}
\usepackage{verbatim}
\usepackage[noend]{algpseudocode}

\usepackage{breakcites}

\usepackage{xspace}

\renewcommand{\paragraph}[1]{\par\noindent\textbf{#1}\hspace{5pt}}

\usepackage[scaled=.88]{helvet}

\newcommand{\algcomment}[1]{\textcolor{blue!70!black}{\footnotesize{\texttt{\textbf{//
					#1}}}}}

\usepackage{nicefrac}

\newcommand{\Thetastar}{\Theta^{\star}}

\newcommand{\fstar}{f^{\star}}
\newcommand{\fhat}{\wh{f}}

\newcommand{\Radoff}{\mathfrak{R}^{\mathsf{o}}_{n}}

\newcommand{\ges}{g.e.s.\xspace}

\newcommand{\radius}{\rho}

\newcommand{\relu}{\textsf{relu}}

\newcommand{\zeroes}{\mb{0}}

\newcommand{\Rk}{R_{K,\rho}}

\newcommand{\glmtron}{GLMtron\xspace}

\newcommand{\normgen}{\square}
\newcommand{\opgen}[1][\square]{\op(#1)}

\renewcommand{\ind}[1]{^{(#1)}}

\newcommand{\bigO}{\cO}
\newcommand{\bigOtilde}{\tilde{\cO}}

\renewcommand{\trn}{\top}
\newcommand{\psdleq}{\preceq}
\newcommand{\psdgeq}{\succeq}

\newcommand{\psdgt}{\succ}
\newcommand{\approxleq}{\lesssim}

\newcommand{\op}{\mathsf{op}}

\newcommand{\subg}{\mathsf{subG}}

\newcommand{\trace}{\mathsf{tr}}

\newcommand{\Mc}{\mathcal{M}}
\newcommand{\Fc}{\mathcal{F}}

\newcommand{\Nc}{\mathcal{N}}

\newcommand{\la}{\langle}
\newcommand{\ra}{\rangle}
\newcommand{\Thetat}{\Theta^{(t)}}
\newcommand{\Thetatt}{\Theta^{(t+1)}}

\newcommand{\Ex}{\mathbb{E}}
\newcommand{\Rb}{\mathbb{R}}
\newcommand{\Thetas}{\Theta^{\star}}
\newcommand{\xx}{\sum_{i=1}^n x_i x_i^{\trn}}

\newcommand{\err}{\mu}

\newcommand{\lnorm}{\left\|}
\newcommand{\rnorm}{\right\|}
\newcommand{\lbr}{\left[}
\newcommand{\rbr}{\right]}

\newcommand{\lr}{\eta_t}

\title{\LARGE Learning nonlinear dynamical systems from a single
  trajectory}
\author{{\large Dylan J. Foster\quad\quad{}Alexander Rakhlin\quad\quad{}Tuhin Sarkar}\vspace{.25em}\\
{\normalsize Massachusetts Institute of Technology}\vspace{.25em}\\
\texttt{\normalsize\{dylanf,\,rakhlin,\,tsarkar\}@mit.edu}}
\date{}

\begin{document}
	\maketitle
	\begin{abstract}

          We introduce algorithms for learning nonlinear dynamical systems of the
		form $x_{t+1}=\sigma(\Thetastar{}x_t)+\veps_t$, where $\Thetastar$ is a weight
		matrix, $\sigma$ is a nonlinear link function, and
		$\veps_t$ is a mean-zero noise process. We
		give an algorithm that recovers the weight matrix $\Thetastar$ from a single trajectory
		with optimal sample complexity and linear running time. The algorithm
		succeeds under weaker statistical assumptions than in previous work, and in
		particular i) does not require a bound on the spectral norm of the weight
		matrix $\Thetastar$ (rather, it depends on a generalization of the
		spectral radius) and ii) enjoys guarantees for
                non-strictly-increasing link
                functions such as the ReLU. Our analysis has two key
		components: i) we give a general recipe whereby global stability for
		nonlinear dynamical systems can be used to certify that the state-vector covariance is well-conditioned, and
		ii) using these tools, we extend well-known algorithms for efficiently learning generalized
		linear models to the dependent setting.
              \end{abstract}

	\section{Introduction}
	\label{sec:intro}

We consider nonlinear dynamical systems of the form
\begin{equation}
x_{i+1} = f^{\star}(x_i) + \veps_i,\label{eq:system}
\end{equation}
where $\fstar:\bbR^{d}\to\bbR^{d}$ is an unknown function, $\crl*{\veps_{i}}_{i=0}^{n}$ is an independent, mean-zero noise
process in $\bbR^{d}$ and $x_0=\zeroes$. Dynamical systems are ubiquitous in applied mathematics, engineering,
and computer science, with applications including control systems,
time series analysis, econometrics, and natural language
processing. The recent success of
deep reinforcement learning
\citep{mnih2015human,silver2017mastering,lillicrap2015continuous} has
led to renewed interest in developing efficient algorithms for
learning complex nonlinear systems such as \pref{eq:system} from data.

In this paper, we focus on the task of estimating the dynamics
$\fstar$ given a single trajectory $\crl*{x_i}_{i=1}^{n+1}$, where
$\fstar$ belongs to a known function class $\cF$. We focus on the
following questions:
\begin{itemize}
	\item What is the sample complexity of recovering the dynamics $\fstar$? How is it determined by $\cF$?
	\item What algorithmic principles enable computationally efficient recovery of
	the dynamics?
\end{itemize}
For linear dynamical systems where $f^{\star}(x)=\Thetastar{}x$, subroutines for efficiently estimating
dynamics from data form a core building block of
\emph{certainty-equivalent} control, which enjoys optimal sample complexity guarantees for this simple setting
\citep{mania2019certainty,simchowitz2020naive}. While linear dynamical
systems have been the subject of intense recent interest
\citep{dean2017sample,hazan2017learning,tu2017least,hazan2018spectral,simchowitz2018learning,sarkar2018fast,simchowitz2019learning,mania2019certainty,sarkar2019fast},
nonlinear dynamical systems are comparatively poorly understood.

\subsection{On the performance of least squares}
On the algorithmic side, a natural starting point for learning the system \pref{eq:system} is the least squares estimator
\begin{equation}
\label{eq:least_squares}
\wh{f}_n = \argmin_{f\in\cF}\frac{1}{n}\sum_{i=1}^{n}\nrm*{f(x_i)-x_{i+1}}^{2},
\end{equation}
where $\nrm*{\cdot}$ denotes the entrywise $\ls_2$ norm. A basic
observation is that the in-sample prediction error (or, denoising error) of this
estimator is bounded by the so-called \emph{offset Rademacher
	complexity} introduced by \cite{rakhlin2014online,liang2015learning}:
\begin{proposition}
	\label{prop:offset_rad}
	The least-squares estimator \pref{eq:least_squares} guarantees
	\begin{equation}
	\En_{\veps}\brk*{\frac{1}{n}\sum_{i=1}^{n}\nrm*{\fhat_n(x_i)-\fstar(x_i)}^{2}} \leq{}
	\En_{\veps}\sup_{g\in\cG}\brk*{\frac{1}{n}\sum_{i=1}^{n} 4\tri*{\veps_i,g(x_i)}-\nrm*{g(x_i)}^{2}}
	\rdef \Radoff(\cG)
	,\label{eq:offset_rad}
	\end{equation}
	where $\cG=\cF-\fstar$ and $\En_{\veps}$ denotes expectation with
	respect to $\crl*{\veps_i}_{i=1}^{n}$.\footnote{Note that $x_i$ is measurable with respect to the $\sigma$-algebra $\sigma(\veps_1,\ldots,\veps_{i-1})$.}
\end{proposition}
The proof is a simple consequence of the basic inequality for least squares \citep{Sara00}. The offset Rademacher process captures the notion of
localization/self-normalization \citep{bartlett2005local,koltchinskii2006local,pena2008self}: The negative quadratic term penalizes
fluctuations from the term involving the random variables
$\crl*{\veps_{i}}_{i=1}^{n}$, leading to fast rates for prediction
error. In particular, if $\veps_i$ has subgaussian parameter
$\tau$\footnote{See \pref{sec:notation}.} and $f^{\star}(x)=\Thetastar{}x$ is a familiar linear dynamical sytem, we have $\Radoff(\cG^{\mathrm{linear}}) \lesssim \tau^{2}\cdot{}\frac{d^{2}}{n}$. The utility of this approach, however, lies in the fact that it easily extends beyond the linear setting. For example, if $\cG$ consists of a class of \emph{generalized linear} dynamical systems of the form
\begin{equation}
x_{i+1}=\sigma(\Thetastar{}x_i) + \veps_{i},\label{eq:glm}
\end{equation}
where $\sigma:\bbR^{d}\to\bbR^{d}$ is a $1$-Lipschitz link function,
we enjoy a similar guarantee: $\Radoff(\cG^{\mathrm{glm}}) \lesssim{}
\tau^{2}\cdot{}\frac{d^{2}}{n}$. More generally, even though \pref{eq:offset_rad} has a complex
dependent structure (the variables $\veps_1,\ldots,\veps_n$
determine the evolution of $x_{1},\ldots,x_{n}$ via \pref{eq:system}), it is possible to bound
the value for general function classes $\cF$ such as neural networks,
kernels, decision trees using sequential covering numbers and chaining
techniques introduced in \cite{RakSriTew14jmlr,rakhlin2014online}. However, there are number of important questions that
remain if one wishes to use this type of learning guarantee for real-world control applications.
\begin{itemize}
	\item \emph{Efficient algorithms.} Even for simple nonlinear
          systems such as the generalized linear model \pref{eq:glm},
          computing the least-squares estimator
          \pref{eq:least_squares} may be computationally intractable
          in the worst case. For what classes of interest can we obtain algorithms that are both computationally efficient \emph{and} sample-efficient?
	\item \emph{Out-of-sample performance.} The prediction error guarantee \pref{eq:offset_rad} only concerns performance on the realized sequence $\crl*{x_{t}}_{t=1}^{n}$. For control applications such as certainty-equivalent control, it is essential to bound the performance of the estimator $\fhat_n$ on counterfactual sequences in which the data generating process is $x_{t+1}=\fhat_n(x_{t})+\veps_{t}$ (i.e., error in simulation). For linear and generalized linear systems, a sufficient condition for such a guarantee is to recover the weight matrix $\Thetastar$ in parameter norm. Under what conditions on the data generating process can we obtain such guarantees?
\end{itemize}

\subsection{Contributions.}
We provide a new efficient algorithm for recovery of \emph{generalized
  linear systems} \pref{eq:glm}. Our algorithm runs in nearly-linear
time and obtains optimal $O\prn*{\sqrt{d^2/n}}$ sample complexity for
recovery in Frobenius norm. Conceptually, our key technical observations are as follows:
\begin{itemize}
	\item We provide a general recipe based on Lyapunov functions for proving that data remains well-conditioned/nearly isometric for stable dynamical systems, without assuming linearity.
	\item We show that efficient algorithms for learning generalized linear models in the i.i.d. setting \citep{kalai2009isotron,kakade2011efficient} cleanly port to the dependent setting. Here the key insight is that the empirical counterparts of simple non-convex
	losses arising from generalized linear models remain well-behaved even under dependent data.
      \end{itemize}
      Our algorithm improves prior work on two fronts:
      First, we do not require a bound on the spectral norm of
      $\Thetastar$, and instead require a bound on a parameter that
      generalizes the notion of the spectral radius to the nonlinear
      setting. Second, we can recover $\Thetastar$ even when the link
      function $\sigma$ is the ReLU, eschewing invertibility
      assumptions from previous results.


	\subsection{Related work}
	\label{sec:related_work}
	Learning guarantees for autoregressive processes have a long history in statistics, though early results for nonlinear systems mainly concern prediction error as in \pref{eq:offset_rad}, and do not consider algorithmic issues \citep{baraud2001adaptive,van2002hoeffding}.

Generalized linear systems \pref{eq:glm} subsume linear dynamical systems, which are fundamental topic in control theory. System identification for the linear setting has been studied since the early days of control \citep{aastrom1971system,ljung1998system,campi2002finite,vidyasagar2006learning}, and is closely related to LQR control. We build on a recent line of work providing non-asymptotic/finite-sample guarantees for the LQR, both for system identification \citep{dean2017sample, tu2017least, simchowitz2018learning,sarkar2018fast,simchowitz2019learning,sarkar2019fast} and (offline and online) control \citep{abbasi2011regret,dean2017sample,dean2018regret,mania2019certainty}. These approaches leverage the rich structure available in the linear setting (in particular, the Riccati equations). While we cannot take advantage of such structure, our Lyapunov approach to establishing well-conditioned empirical designs may be thought of as a natural extension of these structural results to the generalized linear setting \pref{eq:glm}.

Our results are closely related to recent work of \citep{oymak2018stochastic,bahmani2019convex,sattar2020non}. In particular, the concurrent work of \cite{sattar2020non} considers a more general setting and provides guarantees very similar to our own using complementary techniques; we provide a detailed comparison at the end of \pref{sec:alg_overview}. The results of \cite{oymak2018stochastic} and \cite{bahmani2019convex} consider a slightly different form of generalized linear dynamical system inspired by recurrent neural networks, which takes the form $x_{t+1} = \sigma(A^{\star}x_{t} + B^{\star}u_t)$, where $u_{t}$ is an observed noise process (representing a control signal) and the link function $\sigma$ is known and invertible. The key difference between this setup and our own is that because $\sigma$ is invertible, and because both $\crl*{x_t}$ and $\crl*{u_t}$ are observed, the problem reduces to noiseless linear regression: $\sigma^{-1}(x_{t+1}) = A^{\star}x_{t} + B^{\star}u_t$. In particular, since the regression problem is noiseless, their sample complexity guarantees allow for exact recovery once the number of samples reaches a critical threshold; in contrast, for our noisy setting, only approximate recovery is possible given finite samples.

Lastly, we mention that our setting is related to the work of
\cite{hall2018learning} for learning sparse generalized linear autoregressive processes. These results rely on fairly strong assumptions on the data generating process ($x_{t+1}$ is conditionally Poisson or Bernoulli given $\Thetastar{}x_t$), whereas our results, owing to the Lyapunov approach, work for fairly general classes of systems and noise processes.
\subsection{Notation.}\label{sec:notation} Throughout this paper we use $c>0$, $c'>0$, and
$c''>0$ to denote absolute numerical constants whose value may vary
depending on context. We use non-asymptotic big-oh notation, writing
$f=\bigO(g)$ if there exists a numerical constant such that
$f(x)\leq{}c\cdot{}g(x)$ and $f=\bigOtilde(g)$ if
$f\leq{}c\cdot{}g\max\crl*{\log{}g,1}$. We say a random vector $v\in\bbR^{d}$ is subgaussian with variance proxy $\tau^{2}$ if $
\sup_{\nrm{\theta}=1}\sup_{p\geq 1}\left\{p^{-1/2}\left(\Ex[\left|\la v, \theta \ra\right|^{p}]\right)^{1/p}\right\} = \tau$ and $\Ex[v] = \zeroes$, and we denote this by $v\sim\subg(\tau^{2})$. We
let $\nrm*{\cdot}_{\op}$ denote the spectral norm and
$\nrm*{\cdot}_{F}$ denote the Frobenius norm. For a convex set $X$, we let $\mathrm{Proj}_{X}(\cdot)$ denote euclidean projection onto the set. Unless otherwise stated, all dynamical systems considered in this
paper are assumed to start from $x_0=\zeroes$.


	\section{Stability, Lower Isometry, and Recovery}
	\label{sec:stability}
        \label{sec:known_link}

\newcommand{\Constf}{C_{f}}
\newcommand{\empcov}{\wh{\Sigma}_n}

Well-conditioned data plays a fundamental role in statistical
estimation. For linear regression, it is well-known that the minimax
rates for parameter recovery are governed by the spectrum of empirical
design matrix $X_n\in\bbR^{n\times{}d}$ formed by stacking
$x_1,\ldots,x_n$ as rows \citep{hastie2015statistical,wainwright2019high}). In particular, letting
$\empcov=X_n^{\trn}X_n$ denote the empirical covariance, a sufficient
condition for recovery is the \emph{lower isometry} property
\begin{equation}
 \empcov \psdgeq{}\tfrac{1}{2}I;\label{eq:lower_isometry}
\end{equation}
See \cite{lecue2018regularization} for a
contemporary discussion. In this section, we develop tools for proving
lower isometry guarantees for nonlinear dynamical systems such as
\pref{eq:system}. To begin, we make a mild assumption on the noise process.
\begin{assumption}
	\label{ass:noise}
	The noise variables $\crl*{\veps_{t}}_{t=1}^{n}$ are
	independent. Each increment is isotropic (zero-mean, with
	$\En\brk*{\veps_i\veps_i^{\trn}}=I$) and satisfies
	$\veps_{t}\sim\subg(\tau^{2})$.\footnote{The assumption that the noise process $\veps_{t}$ has
		identity covariance serves only to keep notation compact; our results
		transparently extend to general covariance $\Sigma$ under the
		standard assumption that $\Sigma^{-1}\veps_t$ is
		subgaussian. Likewise, our results extend the dependent setting as long as each
		increment is still conditionally mean-zero and subgaussian.}
\end{assumption}
While \pref{ass:noise} ensures that each increment $\veps_i$ is
well-behaved, it is
not clear a-priori whether the empirical design matrix should enjoy
favorable conditioning---indeed, the observations
$\crl*{x_i}_{i=1}^{n+1}$ evolve from the noise process in a complex
dependent fashion. In general, the behavior of the empirical
design matrix will heavily depend on the system $\fstar$. Here we show
that classical results in control theory on \emph{exponential stability} of
the system $\fstar$ provide sufficient conditions for both
upper and lower control of the spectrum of the empirical design
matrix. While our guarantees apply to the noisy system
\pref{eq:system}, our assumptions depend on the behavior of the system
in absence of noise:
\begin{equation}
  \label{eq:noiseless}
  x_{t+1}=f(x_t),
\end{equation}
where $x_0=\zeroes$.
\begin{definition}[Global exponential stability]
	\label{def:ges}
	A noiseless system \pref{eq:noiseless} given by map $f : \Rb^d \rightarrow \Rb^d$
	is globally exponentially stable (\ges) with respect to a norm
	$\nrm*{\cdot}_{\normgen}$ if there exist constants $C_f > 0$
	and  $\rho_f < 1$ depending only on $f$ such that for all $k\geq{}1$,
	\begin{equation}
	\label{eq:ges}
	\|f^{k}\|_{\opgen} \leq C_f\rho_f^{k},
	\end{equation}
	where $\|f\|_{\opgen} \ldef \sup_{x\in\bbR^{d}}
	\frac{\nrm*{f(x)}_{\normgen}}{\nrm*{x}_{\normgen}}$.\footnote{When $f:\bbR^{d}\to\bbR^{d}$, we let $f^{k}$ denote the $k$-times composition of $f$, i.e. $f^{k}=\underbrace{f\circ\cdots{}\circ{}f}_{\text{$k$ times}}$.}
\end{definition}
In this paper, we focus on systems where $\fstar$ satisfies the \ges
property, and where this is certified by a quadratic Lyapunov function.
\begin{definition}
	\label{def:krho_ges}
	A map $f:\bbR^{d}\to\bbR^{d}$ is 
	$(K,\rho)$-\ges if there exists a matrix $K\psdgt{}0$ and constant
	$0\leq{}\rho<1$ such that for all $x\in\bbR^{d}$,
	\begin{equation}
	\label{eq:krho}
	\nrm*{f(x)}_{K}^{2} \leq{} \rho\cdot\nrm*{x}_{K}^{2},
	\end{equation}
	where $\nrm*{x}_{K}\ldef\sqrt{\tri*{x,Kx}}$.
      \end{definition}
Any $(K,\rho)$-\ges map satisfies \pref{eq:ges} with
$\nrm*{\cdot}_{\normgen}=\nrm*{\cdot}_{K}$, $C_f=1$, and
$\rho_f=\rho^{1/2}$. The equation
\pref{eq:krho} is homogeneous under rescaling, and consequently we will assume without loss of
generality that $K\psdgeq{}I$ for the remainder of the paper.

In general, finding certificates of stability for nonlinear dynamical
systems is a difficult problem. Providing necessary and sufficient
conditions for stability for rich classes of nonlinear dynamical
systems remains an active area of research, with most development
proceeding on a fairly case-by-case basis. We develop a general
reduction from lower and upper isometry to $(K,\rho)$-stability, which allows us to leverage
developments in control in a black-box fashion as opposed to having to
prove concentration results case-by-case. Our main result here is
\pref{thm:ges_stable}, which shows that any $(K,\rho)$-\ges system
enjoys both upper and lower isometry.
\begin{theorem}
  \label{thm:ges_stable}
  Consider the noisy system \pref{eq:system}, and let noise process
  satisfy \pref{ass:noise}. Suppose the map $\fstar$
satisfies the $(K,\rho)$-\ges property \pref{def:krho_ges} in the absence of noise. Then for any $\delta>0$,
	once $n \geq{}
	cd\cdot\tfrac{\tau^{4}}{(1-\rho)^{2}}\log(R_{K,\rho}/\delta+1)$,
	with probability at least $1-\delta$ the iterates
        $\crl*{x_i}_{i=1}^{n}$ of the noisy
        system satisfy
	\begin{equation}
	\label{eq:isometry}
	\frac{1}{4}\cdot{}I\psdleq{}\frac{1}{n}\sum_{i=1}^{n}x_ix_i^{\trn}
	\psdleq 4R_{K,\rho}\cdot{}I,
	\end{equation}
	where $R_{K,\rho}\ldef\tfrac{\trace(K)}{1-\rho}$ is the effective
	radius of the system and $c>0$ is an absolute constant.
\end{theorem}
The key feature of \pref{thm:ges_stable} is that we only need to
assume the $(K,\rho)$-\ges property on the map $\fstar$ in the absence
of noise, yet the theorem gives a guarantee on the
trajectory generated by the noisy system \pref{eq:system} as long as
\pref{ass:noise} is satisfied.

The proof
has three parts, each of which relies on the machinery of self-normalization. We first use the structure of the
dynamics \pref{eq:system} to show that the lower isometry in
\pref{eq:isometry} holds as soon as we have a weak \emph{upper} bound
on the covariance of the form $\frac{1}{n}\xx \preceq
\frac{B}{\delta}\cdot{}I$, where $\delta$ is the failure
probability. We then show that in
$(K,\rho)$-\ges systems, this condition is satisfied with
$B=R_{K,\rho}$. Finally, the strong upper bound in
\pref{eq:isometry} is attained by using a self-normalized inequality
to boost the weak upper bound and remove the $1/\delta$
factor.

\subsection{Lower isometry for generalized linear systems}

We now provide sufficient conditions under which the generalized
linear systems that are the focus of our main learning results satisfy the \ges
property. We make the following mild regularity assumption on the link function.
\begin{assumption}
  \label{ass:lipschitz_link}
The link function $\sigma:\bbR^{d}\to\bbR^{d}$ has the form $\sigma(x)
\ldef  (\sigma_1(x_1), \hdots, \sigma_d(x_d))$, where each coordinate function $\sigma_i:\bbR\to\bbR$ is non-decreasing,
	$1$-Lipschitz, and satisfies $\sigma_i(0)=0$. 
      \end{assumption}
With this assumption, the following constrained Lyapunov equation
provides a sufficient condition under which the generalized linear system
satisfies the \ges property.
\begin{proposition}
	\label{prop:stability_theta}
	Suppose there exists a \emph{diagonal} matrix $K \psdgt 0$ and scalar
	$\rho<1$ such that
	\begin{equation}
	\label{eq:lyapunov}
	\Theta^{\trn}K\Theta \psdleq{} \rho\cdot{}K.
	\end{equation}
	Then the map $f = \sigma \circ
	\Theta$ is $(K,\rho)$-\ges whenever $\sigma$ satisfies \pref{ass:lipschitz_link}.
      \end{proposition}
      \begin{proof}
	Observe that for any $x\in\bbR^{d}$ we have
	\[
	\nrm*{f(x)}_{K}^{2}
	= \nrm*{\sigma\prn*{\Theta{}x}}_{K}^{2}
	\overset{(i)}{\leq{}}
	\nrm*{\Theta{}x}_{K}^{2}
	\overset{(ii)}{\leq{}} \rho\cdot\nrm*{x}^{2}_{K},
	\]
	where (i) uses that $K$ is diagonal and positive definite and
	that each coordinate-wise link $\sigma_i:\bbR\to\bbR$ is
	$1$-Lipschitz with $\sigma_i(0)=0$, and (ii) uses the Lyapunov equation
	\pref{eq:lyapunov}.
\end{proof}
\pref{prop:stability_theta}
can be used to invoke \pref{thm:ges_stable} for any generalized linear
system of the form $\fstar =
\sigma \circ \Theta$. Thus, we can ensure lower
and upper isometry hold for generalized linear systems whenever
their stability is certified the Lyapunov condition
\pref{eq:lyapunov}.

The equation \pref{eq:lyapunov} strengthens the usual Lyapunov
condition for linear systems by adding the additional constraint that
$K$ is diagonal. This condition is stronger than the
classical spectral radius condition that $\rho(\Theta)<1$,
but it can easily be seen that some type of strengthening is necessary, as
the classical condition is not sufficient for nonlinear systems. For example, the matrix
$\Theta = \prn*{
	\begin{smallmatrix}
	1&1\\
	-1&-1
	\end{smallmatrix}
}$ has $\radius(\Theta)=0$, but the map
$x\mapsto{}\relu(\Theta{}x)$ is not \ges---indeed, we have $\relu(\Theta{}e_1) = e_1$, where $e_1$
is the first standard basis vector. A sufficient condition for
\pref{prop:stability_theta} is that $\Theta$ has spectral
norm bounded by unity, but the condition
\pref{eq:lyapunov} is a strictly weaker than this assumption. Further
sufficient conditions include: 1) $\rho(\Theta)<1$ and $\Theta$ has
non-negative entries \citep[Proposition 2]{rantzer2011distributed},
and 2) $\rho\prn*{\abs*{\Theta}}<1$, where $\abs*{\Theta}$ denotes
element-wise absolute value operator.

To close this section, we remark that for any fixed link
function, the Lyapunov condition \pref{eq:lyapunov} may be overly
pessimistic as a condition for exponential stability, but a
classical line of research in nonlinear control establishes that
this condition is actually \emph{necessary} for a somewhat more
general class of nonlinearities
\citep{megretski1993power2,shamma1994robust,poola1995robust}. Nonetheless,
finding tighter conditions for specific link functions of interest such as the
ReLU remains an interesting direction for future work, as does
leveraging more general (e.g., piecewise) Lyapunov functions.


	\section{Algorithms for generalized linear dynamical systems}
	\label{sec:learning}

We now leverage the isometry results of \pref{sec:stability} to
develop efficient algorithms with parameter recovery guarantees for
generalized linear systems. Following \pref{sec:stability}, we make the
following assumption on the generalized linear system.
\begin{assumption}
	\label{ass:dynamics}
	The system \pref{eq:system} is generalized
	linear ($\fstar=\sigma\circ\Thetastar$) and is $(K,\rho)$-\ges in
	the sense of \pref{prop:stability_theta}. Furthermore,
	$\nrm*{\Thetastar}_{F}\leq{}W$, where $W$ is known to the learner.\footnote{There is no restriction on
		the range of the parameter $W$, but some of our sample
		complexity guarantees depend on it polynomially.}
\end{assumption}

\paragraph{Background: Learning generalized linear models}
Our algorithm for learning generalized linear dynamical systems builds
on developments for learning generalized linear models in statistical
learning. Consider the simpler setting where we receive
$\crl*{(x_i,y_i)}_{i=1}^{n}$ i.i.d., where
  $y=\sigma(\tri*{\theta^{\star},x})+\veps$ and
  $\En\brk*{\veps\mid{}x}=0$. For this setting the population loss
  $L(\theta)\ldef{}\En_{x,y}\brk*{\prn*{\sigma(\tri{\theta,x})-y}^{2}}$
  is not convex. However, if the link function $\sigma$ is strictly
  increasing the and the population covariance $\En\brk*{xx^{\trn}}$
  is well-conditioned, the loss satisfies a gradient-dominance type
  property, and gradient descent on the empirical loss will converge
  to $\theta^{\star}$ given sufficiently many samples
  \citep{mei2018landscape}. To provide guarantees even when $\sigma$
  is not strictly increasing, we opt to use a variant of the \glmtron
  algorithm introduced by \cite{kakade2011efficient}. The \glmtron algorithm
  performs gradient descent using a ``pseudogradient'' for the
  empirical loss in which the derivative of $\sigma$ is simply dropped:
  \[
  \theta^{(t+1)} = \theta^{(t)} - \frac{1}{n} \sum_{i=1}^n \prn*{\sigma\prn{\la
      \theta^{(t)}, x_i \ra } - y_i}x_i
\]
Following the pseudogradient allows the algorithm to efficiently
provide prediction error guarantees even when $\sigma$ is not strictly
increasing, and this is the starting point for our approach.

\subsection{Algorithm and guarantees}
\label{sec:alg_overview}
\pref{alg:parameter_estimates} is a natural extension of \glmtron
to handle the vector-valued target variables and
matrix-valued parameters that arise in our dynamical system setting.
\begin{algorithm}[h]
	\caption{Parameter estimation for generalized linear systems}
	\label{alg:parameter_estimates}
	\begin{algorithmic}[1]
		\State\textbf{input}: Single trajectory: $X_{n+1}= \{x_i\}_{i=1}^{n+1}$, Learning rate schedule: $\eta_{t}$.
		\State\textbf{initialize}: $\wh{\Theta}^{(1)} =
		\textbf{0}_{d \times d}$.
		\State Define $\cM=\crl*{\Theta\in\bbR^{d\times{}d}\mid\nrm*{\Theta}_{F}\leq{}W}$.
		\For{$t = 1, \ldots, m$} 
		\State $\wh{\Theta}^{(t+1)} =
                \text{Proj}_{\Mc}(\wh{\Theta}^{(t)} - \eta_t
                \hG(\wh{\Theta}^{(t)},
                X_{n+1}))$.~~~~~~\algcomment{$\hG$ is the
                  pseudogradient; see \pref{eq:pseudogradient}.}
		\EndFor
		\State\label{line:last_line} \textbf{return}:
                $\wh{\Theta}=\wh{\Theta}\ind{m}$ (Option I), or
                $\wh{\Theta}=\wh{\Theta}\ind{t}$ with $t\in\brk{m}$
                uniform (Option II).
	\end{algorithmic}
\end{algorithm}

\pref{alg:parameter_estimates} is closely related to projected gradient descent
on the empirical square loss
$\wh{L}(\Theta,X_{n+1})\ldef{}\frac{1}{n}\sum_{t=1}^{n}\nrm*{\sigma(\Theta{}x_t)-x_{t+1}}^{2}$, but rather than following the gradient, the
algorithm follows the pseudogradient
\begin{equation}
\hG(\Theta^{(t)}, X_{n+1}) \ldef \frac{1}{n} \sum_{i=1}^n
(\sigma(\Theta^{(t)} x_{i}) - x_{i+1})x_i^{\top},\label{eq:pseudogradient}
\end{equation}
attained by dropping the link derivative $\sigma'$ from the gradient. This
modification allows for prediction guarantees without assuming a lower
bound on the link function derivative, and allows for weaker
dependence on the derivative lower bound for parameter recovery
guarantees. In particular, we show that the algorithm obtains the best of both
worlds in a certain sense. First, with only the assumption that the link function is
Lipschitz, the algorithm ensures that iterates have low prediction
error on average. Consequently, if we select an iterate uniformly at
random (Option II in \pref{alg:parameter_estimates}), the iterate will
have low prediction error  (a ``slow rate'' of type $1/\sqrt{n}$) in expectation. On the other hand, suppose the
following assumption holds.
\begin{assumption}
	\label{ass:known_link}
	There exists a constant $\zeta>0$ such that for all $i$,
	$\abs*{\sigma_i(x)-\sigma_i(y)}\geq{}\zeta\abs*{x-y}$ for all $x,y\in\bbR$.
      \end{assumption}
      In this case, the algorithm enjoys linear convergence, and
      taking the last iterate (Option I in
      \pref{alg:parameter_estimates}) leads to a ``fast'' $1/n$-type rate
      for prediction error, as well as a parameter recovery guarantee.

      To state the performance guarantee, we let $\cE(\Theta) \ldef \frac{1}{n}\sum_{i=1}^n
\|  (\sigma(\Theta x_i) - \sigma(\Thetas x_i)) \|^2$ denote the
in-sample prediction error, and let $\En_{\cA}$ denote expectation with respect to the algorithm's
internal randomness (uniform selection of the iterate returned
in \pref{line:last_line} under Option II).
\begin{theorem}
	\label{thm:known_link_main}
	Let $\delta>0$ be fixed and let
        \savehyperref{ass:noise}{Assumptions \ref*{ass:noise}-}\ref{ass:dynamics} hold. Whenever $n \geq
	\frac{c\tau^4d}{1-\rho}\log{\left(\Rk/\delta+1\right)}$,
	\pref{alg:parameter_estimates} enjoys the following
        guarantees:
	\begin{enumerate}
		\item \textbf{Slow rate.} If
                  $\eta_t=\frac{1}{16R_{K,\rho}}$ and
                  $m\geq{}C_0 \cdot{}\sqrt{n}$, then with probability
		at least $1-\delta$, \pref{alg:parameter_estimates}
                  with Option II has
		\begin{equation}
		\label{eq:thm2_nondecreasing}
		\En_{\cA}\brk*{\cE(\wh{\Theta})} \leq{} C_{1}\cdot{}\sqrt{\frac{d^2}{n}\log{\prn*{4\Rk/\delta + 1}}},
		\end{equation}
		where $C_0 \leq c \cdot{}WB(\tau^2 d \Rk \log{\prn*{4\Rk/\delta+ 1}})^{-1/2}$ and $C_1 \leq c \tau W \sqrt{\frac{\sigma_{\max}(K)}{1-\rho}}$.
		\item\textbf{Fast rate.} Suppose that
                  \pref{ass:known_link} holds in addition to
                  \savehyperref{ass:noise}{Assumptions
                    \ref*{ass:noise}-}\ref{ass:dynamics}. If
                  $\eta_t=\frac{\zeta^{2}}{(16R_{K,\rho})^2}$ and
                  $m\geq{}C_2 \cdot{}\log{\prn*{1+\frac{n W^2B^2}{\tau^2 \Rk}}}$, then with probability
                  at least $1-\delta$, \pref{alg:parameter_estimates}
                  with Option I has
		\[
                  {\cE(\wh{\Theta})} \leq{}
		C_{3}\cdot{}{\frac{d^2}{n}\log{\Big(4\Rk/\delta + 1\Big)}}, \quad \text{and} \quad
		{\|\wh{\Theta} - \Thetas\|_F^2} \leq{} C_4
		\cdot{}{\frac{d^2}{n}\log{\Big(4\Rk/\delta+ 1\Big)}},
		\]
		where
		$C_2 \leq{} c B^2 \zeta^{-4}$, $C_3 \leq{} c \tau^2 B^2 \zeta^{-6}\cdot
		\frac{\sigma_{\max}(K)}{(1-\rho)}$ and $C_4\leq{}c \tau^2\zeta^{-4}\cdot
		\frac{\sigma_{\max}(K)}{(1-\rho)}$. 
              \end{enumerate}

      \end{theorem}
\pref{ass:known_link} is satisfied for the so-called ``leaky ReLU''
$\relu_{\beta}(x)\ldef\max\crl*{x,\beta{}x}$ with $\beta>0$, but not
for the ReLU. Our next theorem shows that under stronger assumptions
on the noise process, the algorithm succeeds at parameter recovery for
the ReLU as well. We make the following assumption.
\begin{assumption}
	\label{ass:relu}
	The link function $\sigma$ is the ReLU ($\sigma_i(x_i)=\relu(x_i)\ldef{}\max\crl*{x_i,0}$)
	and the noise process is Gaussian, with
        $\veps_i\sim\cN(0,I)$.
      \end{assumption}
The gaussian assumption ensures for any pair of parameters,
sufficiently large probability mass lies in the region where the ReLU
is active. In particular, we use (\pref{lem:gaussian_relu}) that for
any pair $u,v\in\bbR^{d}$, $\En_{\veps\sim{}\cN(0,I)}[(\relu(\la u, \veps\ra) - \relu(\la v,
\veps\ra))^2] \propto\|u-v\|^2$. Similar guarantees can be established
for log-concave distributions using arguments in
\cite{balcan2013active}, but we consider only the gaussian case for simplicity.
\begin{theorem}[Parameter recovery for the ReLU]
  \label{thm:relu_case}
  Suppose assumptions \savehyperref{ass:noise}{Assumptions
                    \ref*{ass:noise}-}\ref{ass:dynamics} and \pref{ass:relu} hold.
                  Let $\delta>0$ be fixed and suppose $n\geq
                  c\frac{\tau^4d^3}{1-\rho}\log{\left(\Rk/\delta+1\right)}$. Then
                  when $\eta_t=(16\Rk)^{-2} e^{-4 \rho \Rk }$
                  and $m\geq{}C_0 \cdot{} \log{}n$,
                  \pref{alg:parameter_estimates} with Option I
                  guarantees that with probability at least $1-\delta$, 
	\[
	{\|\wh{\Theta} - \Thetas\|_F^2} \leq{} C_{1}\cdot{}\frac{d^2}{n} \cdot{}\Rk^{2}\log^{2}\prn*{2\Rk{}n/\delta+1},
	\]
	where $C_0 \leq cB^2e^{8\rho \Rk}$ and $C_1 \leq \frac{c \tau W^2}{(1-\rho)^2} \cdot{}e^{8\rho \Rk }$.
\end{theorem}
Let us discuss some key features of \pref{thm:known_link_main} and
\pref{thm:relu_case}. First, in the fast rate regime where
\pref{ass:known_link} holds, \pref{thm:known_link_main} attains the
usual parametric rate
$\En\nrm*{\wh{\Theta}-\Thetastar}_{F}\approxleq\sqrt{\frac{d^{2}}{n}}$,
which is optimal for this setting
\citep{tsybakov2008introduction}. The algorithm is also linearly
convergent in this regime, and so the runtime to attain parameter
recovery is nearly linear. On the other hand, the dependence on problem-dependent parameters such as
$K$ and $\rho$ in the results can almost certainly be improved for all
of the results. For
example, while there are certainly systems for which $\nrm*{x_t}$
grows as $\frac{1}{1-\rho}$, it is not clear whether exponential
dependence on this parameter  in \pref{thm:relu_case} is required for parameter recovery with the
ReLU. More generally, the factor $e^{R_{K,\rho}}$ in
\pref{thm:relu_case} can be replaced with
$\max_{i}e^{\nrm*{\mu_i}^{2}}$, where $\mu_i$ is an upper bound on the
(conditional) expected value of $x_i$ at time $i$. Our analysis simply
bounds $\nrm*{\mu_i}^2$ by $R_{K,\rho}$, and any improvements to this
norm bound for systems of interest will immediately lead to improved rates.
\\
\paragraph{Detailed comparison with related work}
Concurrent work of \cite{sattar2020non} also considers the problem of
learning generalized linear systems of the form \pref{eq:glm}, and
provides similar guarantees to \pref{thm:known_link_main} in the fast
rate regime. Let $\theta^{\star}_k$ and $\wh{\theta}_k$
denote the $k^{th}$ rows of $\Thetas$ and $\wh{\Theta}$
respectively. Theorem 6.2 of \citep{sattar2020non} considers a
gradient descent-based estimator and shows that once $n$ is
a sufficiently large problem-dependent constant and the number of
iterations is polylogarithmic in $n$, $\nrm*{\wh{\theta}_k-\theta^{\star}_k}
\leq \bigOtilde\prn*{\frac{c \tau}{\zeta^2} \cdot
  \sqrt{\frac{d}{n}}}$, where $\bigOtilde \prn*{\cdot}$ hides
logarithmic dependence  problem parameters. Under the same conditions,
\pref{thm:known_link_main} attains a comparable guarantee
of $\nrm*{\wh{\Theta}-\Theta^{\star}}_F \leq \bigOtilde\prn*{\frac{c
    \tau}{\zeta^2} \cdot \sqrt{\frac{d^2}{n}}}$. At a more conceptual
level, our techniques and tools are complementary: \cite{sattar2020non} use
mixing time arguments and analyze gradient descent, while we use
martingale arguments and analyze \glmtron. Additional results we
provide include i) explicit Lyapunov conditions under which
stability of the generalized linear system holds ii) prediction
error guarantees for non-strictly increasing link functions, and iii)
parameter recovery guarantees for the ReLU.


	\section{Discussion}
	We have shown that the exponential stability, in conjunction with
	Lyapunov arguments, offers a simple approach to establishing isometry
	guarantees for data generated by nonlinear dynamical
	systems, and we have provided efficient algorithms
	for learning and parameter recovery in generalized linear
        systems. We hope that the analysis techniques introduced here
        will find use beyond the generalized linear setting, as well
        as for end-to-end control.
	
        Going forward, it will be interesting to draw further connections and
	build stronger bridges between Lyapunov theory and empirical
        process theory for dependent data. For example,
	what properties of data generated by dynamical systems can we use Lyapunov
	functions to certify, going beyond lower and upper isometry?

        Lastly, we remark on an extension to the non--autonomous setting. Consider a
        non--autonomous system of the form $x_{i+1} = \sigma(\Theta
        \cdot x_i + B \cdot u_i) + \veps_i$, where $\{u_i\}_{i=1}^n$
        are control inputs. This setting reduces to the autonomous
        case via the expression
	\[
	\begin{bmatrix}
	x_{i+1}\\
	u_{i+1}
	\end{bmatrix} = \sigma \prn*{\begin{bmatrix}\Theta & B \\
		0 & 0
		\end{bmatrix}	\begin{bmatrix}
		x_{i}\\
		u_{i}
		\end{bmatrix}} + 	\begin{bmatrix}
			\veps_{i}\\
			u_{i+1}
		\end{bmatrix},
              \]
              and our techniques can consequently be applied as long
              as the control inputs have persistent excitation.

              \subsection*{Acknowledgements}
              We thank Adam Klivans, Alexandre Megretski, and Karthik Sridharan for
              helpful discussions. We acknowledge the support of ONR award \#N00014-20-1-2336 and
              NSF TRIPODS award \#1740751.

	\newpage
	\bibliography{refs}
	
	\appendix
	
	\section{Basic technical results}
	\label{app:basic_technical_results}
	\begin{lemma}[Freedman's Inequality (e.g., \cite{agarwal2014taming})]
	\label{lem:freedman_lemma}
	Let $\{X_i\}_{i=1}^n$ be a sequence of real--valued random
        variables such that for all $i$, $\abs*{X_i} \leq R$ and $\Ex[X_i \mid X_1, \hdots, X_{i-1}] = 0$. Define $S \ldef \sum_{i=1}^n X_i$ and $V \ldef \sum_{i=1}^n \Ex[X_i^2 \mid X_1, \hdots, X_{i-1}]$. For any $\delta \in (0, 1)$ and $\lambda \in [0, 1/R]$, with probability at least $1-\delta$,
	\[
		|S| \leq (e-2) \lambda V + \frac{\log{(1/\delta)}}{\lambda}.
	\]
      \end{lemma}
\begin{lemma}[\cite{vershynin2010introduction}, Theorem 5.39]
	\label{lem:subgaussian_isometry}
	Let $A$ be a $n \times d$ matrix whose rows,
        $\{A_i\}_{i=1}^n$, are independent and isotropic random
        vectors in $\Rb^{d}$, belonging to $\subg(\tau^2)$. Then for every $t \geq 0$, with probability at least $1- 2\exp{(-c\tau^{-4}t^2)}$, one has 
	\[
	\sqrt{n} - c'\tau^{2}\sqrt{d} - t \leq \sigma_{\min}(A) \leq \sigma_{\max}(A) \leq \sqrt{n} + c'\tau^{2}\sqrt{d} + t,
	\]
	where $c$ and $c'$ are numerical constants.
	
      \end{lemma}
\begin{lemma}
	\label{lem:subgaussian_invertible}
	Let $A$ be a $n \times d$ matrix whose rows,
	$\{A_i\}_{i=1}^n$, are independent and isotropic random
	vectors in $\Rb^{d}$, belonging to $\subg(\tau^2)$. Then whenever 
	\[
	\sqrt{n} \geq c \tau^2\prn*{\sqrt{d}  + \sqrt{\log{(2/\delta)}}},
	\]
        we have that with probability at least $1- \delta$,
	\[
          \frac{3}{4}\cdot{}  I \preceq \sum_{i=1}^d A_i A_i^{\trn} \preceq \frac{5}{4} \cdot{} I,
	\]
	where $c>0$ is a numerical constant.
\end{lemma}
\begin{proof}[\pfref{lem:subgaussian_invertible}]
  Let constants $c$ and $c'$ be as in
  \pref{lem:subgaussian_isometry}. We invoke \pref{lem:subgaussian_isometry} with parameter $t \coloneqq c^{-1} \tau^2 \sqrt{\log{\prn*{2/\delta}}}$,
which implies that
	\[
	\sqrt{n} - c'\tau^{2}\sqrt{d} - c^{-1} \tau^2 \sqrt{\log{\prn*{2/\delta}}} \leq \sigma_{\min}(A) \leq \sigma_{\max}(A) \leq \sqrt{n} + c'\tau^{2}\sqrt{d} + c^{-1} \tau^2 \sqrt{\log{\prn*{2/\delta}}}	
	\]
	with probability at least $1-\delta$. The final bound follows because 
	\[
	\sqrt{n} \geq 16 \cdot{} \prn*{ c'\tau^{2}\sqrt{d}  + c^{-1} \tau^2 \sqrt{\log{(2/\delta)}}}.	
	\]
\end{proof}
      \begin{lemma}[\cite{sarkar2018fast}, Proposition
        7.1]
	\label{lem:psd_result_2}
	Let $P$ and $V$ be arbitrary positive semidefinite and postive
        definite matrices, respectively, and define $\bar{P} = P +
        V$. Let $Q$ be any matrix for which
	\[
	\|\bar{P}^{-1/2} Q\|_{\op} \leq \gamma.
	\]
        Then for any vector $v$, we have
	\[
          \|v^{\trn}Q\| \leq \gamma\sqrt{v^{\trn}Pv + v^{\trn}Vv}.
	\]
\end{lemma}

\begin{lemma}[\cite{sarkar2018fast}, Proposition 3.1]
  	\label{lem:selfnorm_ineq}
	Let $V \succ 0$ be a fixed matrix, and consider the dynamics
	\[
	x_{i+1} = f(x_i)+ \varepsilon_{i},
	\]
where the noise process follows \pref{ass:noise}. Define
$\bar{Y}_{n-1} = \sum_{i=1}^{n-1}f(x_i) f(x_i)^{\trn} + V$. Then for any $0 <  \delta < 1$, with probability at least $1-\delta$,
	\begin{align*}
          \nrm*{ (\bar{Y}_{n-1})^{-1/2} \sum_{i=0}^{n-1}f( x_{i}) \varepsilon_{i}^{\trn}}_{\op} \leq c\tau \sqrt{d \log{\prn*{\frac{5 \text{det}(\bar{Y}_{n-1})^{1/2d} \text{det}(V)^{-1/2d}}{\delta^{1/d}}}}},
	\end{align*}
        where $c$ is an absolute constant.\footnote{This lemma is stated in \cite{sarkar2018fast} for
          linear systems, but one can easily verify that it holds for
          arbitrary systems as stated here.}
      \end{lemma}


	\section{Proofs from \pref{sec:intro}}
	\label{app:intro_proofs}

\begin{proof}[\pfref{prop:offset_rad}]
Optimality of $\wh{f}_n$ implies that
$\sum_{i=1}^{n}\nrm*{\wh{f}_n(x_i)-x_{i+1}}^{2}\leq
\sum_{i=1}^{n}\nrm*{f^*(x_i)-x_{i+1}}^{2}$, and consequently
$$\sum_{i=1}^n \nrm*{\wh{f}_n(x_i)-f^*(x_i)}^2 \leq 2\sum_{i=1}^n \tri*{\veps_i, \wh{f}_n(x_i)-f^*(x_i)}.$$
Rearranging, we have
\begin{align*}
	\sum_{i=1}^n \nrm*{\wh{f}_n(x_i)-f^*(x_i)}^2 &\leq 4\sum_{i=1}^n \tri*{\veps_i, \wh{f}_n(x_i)-f^*(x_i)} - \sum_{i=1}^n \nrm*{\wh{f}_n(x_i)-f^*(x_i)}^2, \\
	&\leq \sup_{f\in\Fc } \sum_{i=1}^n 4\tri*{\veps_i, f(x_i)-f^*(x_i)} - \nrm*{f(x_i)-f^*(x_i)}^2.
\end{align*}
Taking expectation on both sides and dividing by $n$ completes the proof.
\end{proof}


	\section{Proofs from \pref{sec:stability}}
	\label{app:section2_proof}

\newcommand{\deltanot}{\delta_0}

\subsection{Proof of \pref{thm:ges_stable}}
\label{app:ges_proofs}
\begin{lemma}
  \label{lem:sample_cov}
  Suppose there exist constants $B>0$ and $\deltanot>0$ such that with
  probability at least $1-\deltanot$, each trajectory generated by the
  system \pref{eq:system} satisfies
	\[
          \sum_{i=1}^n f(x_i) f(x_i)^{\trn} \preceq \frac{nB}{\deltanot}\cdot{}I,
        \]
        Then there exists
        a numerical constant $c>0$ such that whenever
        $n\geq{}c\cdot{}\tau^{4}d\log(B/\delta_0+1)$, 
        it holds that with probability at least $1-2\delta_0$,
	\[
	\sum_{i=1}^n 	x_{i} x_{i}^{\trn} \succeq \frac{n}{4}\cdot{}I.
      \]
\end{lemma}

\begin{proof}[\pfref{lem:sample_cov}]
From the dynamics \pref{eq:system}, we have deterministic identities
	\begin{align}
		x_{i+1} x_{i+1}^{\trn} &= f(x_i) f(x_i)^{\trn} + f(x_i) \varepsilon_{i}^{\trn} + \varepsilon_{i} f(x_i)^{\trn} + \varepsilon_{i} \varepsilon_{i}^{\trn}, \nonumber \intertext{and}
		\sum_{i=1}^n x_{i+1} x_{i+1}^{\trn} &= \sum_{i=1}^nf(x_i) f(x_i)^{\trn} +  f(x_i) \varepsilon_{i}^{\trn} + \varepsilon_{i} f(x_i)^{\trn} + \varepsilon_{i} \varepsilon_{i}^{\trn}. \label{eq:expansion}
	\end{align}
	We know from~\pref{lem:subgaussian_invertible} that for any $\delta>0$, when $n\geq{}c\tau^{4}(d+\log(2/\delta))$,
        with probability at least $1-\delta$,
	\begin{equation}
	\frac{3}{4}I \preceq \frac{1}{n}\sum_{i=1}^{n} \varepsilon_{i} \varepsilon_{i}^{\trn} \preceq \frac{5 }{4}I. \label{eq:error_event} 
      \end{equation}
Now, define 
\[
P \ldef \sum_{i=1}^{n-1}f(x_i) f(x_i)^{\trn},\quad Q \ldef  \sum_{i=0}^{n-1}f(x_i)
\varepsilon_{i}^{\trn},\quad\text{and}\quad 
V \ldef \frac{3n
	}{4}I.
\]	
To prove the main result, we use \pref{lem:psd_result_2} and
\pref{lem:selfnorm_ineq} to show that the cross terms in
\pref{eq:expansion} have little impact. Specifically, \pref{lem:selfnorm_ineq} states that for any $\delta>0$, with probability at least $1-\delta$, it is ensured that 
\begin{equation}
          \nrm*{ (P+V)^{-1/2} Q}_{\op} \leq c\tau \sqrt{d \log{\prn*{\frac{5 \text{det}(P+V)^{1/2d} \text{det}(V)^{-1/2d}}{\delta^{1/d}}}}}. \label{eq:selfnorm_event}
\end{equation}
Let $\gamma \ldef  c\tau\sqrt{d
          \log{\left(\frac{5 \mathrm{det}(P+V)^{1/2d}
                \mathrm{det}(V)^{-1/2d}}{\delta^{1/d}}\right)}}$. Conditioning
        on the event \pref{eq:selfnorm_event} and using
        \pref{lem:psd_result_2}, we have that for \textit{any} unit vector
        $v \in \Rb^d$, $\lnorm v^{\trn}Q\rnorm \leq \sqrt{\kappa^2 +
          \frac{3n}{4}} \gamma$, where $\kappa^2 = v^{\trn} P
        v$. Substituting this bound into \pref{eq:expansion} and
        conditioning on \pref{eq:error_event} we are guaranteed that
        for any unit vector $v$,
	\begin{align}
	\label{eq:empcov_lb}
		v^{\trn} \sum_{i=1}^n x_{i+1} x_{i+1}^{\trn} v \geq \kappa^2 -  2\sqrt{\kappa^2 + \frac{3n }{4}} \gamma + \frac{3n}{4}.
	\end{align}
	Selecting $\delta=\delta_0/2$ and conditioning on the event $\sum_{i=1}^n f(x_i) f(x_i)^{\trn} \preceq
        \frac{nB}{\delta_0}I$, which happens with probability at least
        $1-\delta_0$, we can upper bound $\gamma$ as 
	$$\gamma \leq c\tau\sqrt{
		\log{\left(\frac{5 \mathrm{det}\left(\frac{n B}{\delta_0}I + \frac{3n}{4}I\right)^{1/2}
				\mathrm{det}\left( \frac{3n}{4}I\right)^{-1/2}}{\delta}\right)}} \leq c\tau\sqrt{
		d\log{\left(\frac{2 B}{\delta_0} + 1\right)} + \log{(1/\delta)}}.$$
Simplifying \pref{eq:empcov_lb} further, we have
	\begin{align*}
		v^{\trn} \sum_{i=1}^n x_{i+1} x_{i+1}^{\trn} v \geq  \kappa^2 - 2\kappa\gamma -  2\sqrt{\frac{3n }{4}} \gamma + \frac{3n }{4} \geq -\gamma^{2} -  2\sqrt{\frac{3n }{4}} \gamma + \frac{3n }{4}.
	\end{align*}
	Thus, whenever $n\geq{}c\gamma^{2}$, i.e.
	\begin{equation}
	\label{eq:requirement}
c\tau\sqrt{
	d\log{\left(\frac{2 B}{\delta_0} + 1\right)} + \log{(1/\delta)}} \leq \sqrt{\frac{n}{12}},
	\end{equation}
	we have
	\[
	v^{\trn} \sum_{i=1}^n x_{i+1} x_{i+1}^{\trn} v \geq \frac{n }{4},
	\]
	for any vector unit $v$ after conditioning on the events \pref{eq:error_event}, \pref{eq:selfnorm_event}, and $\sum_{i=1}^n f(x_i) f(x_i)^{\trn} \preceq
	\frac{nB}{\delta_0}I$. 	The condition on $n$ in the lemma
        statement follows from the requirement in
        \pref{eq:requirement}. Since $v^{\trn} \sum_{i=1}^n x_{i+1}
        x_{i+1}^{\trn} v \geq \frac{n }{4}$ holds for any unit vector -- fixed or depending on $\{x_i\}_{i=1}^n$, we have
	\[
	\sum_{i=1}^n 	x_{i+1} x_{i+1}^{\trn} \succeq \frac{n }{4}I.
      \]
      Note that our development so far requires that events \pref{eq:error_event}, \pref{eq:selfnorm_event} and $\sum_{i=1}^n f(x_i) f(x_i)^{\trn} \preceq
      \frac{nB}{\delta_0}I$ occur simultaneously, which happens with probability at least $1-\delta_0-2\delta$. To
        deduce the theorem statement we set $\delta=\delta_0/2$.
      \end{proof}

\pref{lem:sample_cov} requires that there is some $B$ such that $  \sum_{i=1}^n f(x_i) f(x_i)^{\trn} \preceq \frac{nB}{\delta}I$ with probability at least $1-\delta$. We first show that in $(K,\rho)$-\ges systems, this condition is satisfied with $B=R_{K,\rho}$ (\pref{lem:weak_ub}). This immediately gives the lower isometry bound in \pref{thm:ges_stable}. The upper bound in \pref{thm:ges_stable} is attained through \pref{lem:strong_ub}, which sharpens the upper bound in \pref{lem:weak_ub} by removing the $1/\delta$ factor.
\begin{lemma}[Weak Upper Bound]
	\label{lem:weak_ub}
	Let the noise process satisfy \pref{ass:noise}. For any
        $(K,\rho)$-\ges map $f^{\star}$, with probability
        $1-\delta$, 
	\[
	\sum_{i=1}^n x_i^{\trn}K x_i \leq \frac{n \trace(K)}{(1-\rho)\delta}, \quad\text{and}\quad	\sum_{i=1}^n x_i^{\trn} x_i \leq \frac{n \trace(K)}{(1-\rho)\delta}.
	\]
	Furthermore, with probability at least $1-\delta$,
	\[
	\sum_{i=1}^n \fstar(x_i)^{\trn}K \fstar(x_i) \leq \frac{n \rho \trace(K)}{(1-\rho)\delta}, \quad\text{and}\quad	\sum_{i=1}^n \fstar(x_i)^{\trn} \fstar(x_i) \leq \frac{n \rho \trace(K)}{(1-\rho)\delta}.
	\]
\end{lemma}
\begin{proof}[\pfref{lem:weak_ub}]
	Let $\{x_i\}_{t=1}^n$ be the state observations from the nonlinear dynamical system. Then 
	\begin{align*}
		\sum_{i=1}^n x_i^{\trn} K x_i = \sum_{i=0}^{n-1} \wh{x}_i^{\trn}K\wh{x}_i + 2\sum_{i=0}^{n-1} \varepsilon_{i}^{\trn}K \wh{x}_i + \sum_{i=0}^{n-1} \varepsilon^{\trn}_i K \varepsilon_i
	\end{align*}
	where $\wh{x}_i \ldef \fstar(x_i)$. Using the $(K,\rho)$-\ges{} condition, we can upper bound the
        first term on the right-hand side to get 
	\[
	\sum_{i=1}^n x_i^{\trn} K x_i \leq \rho\sum_{i=0}^{n-1} x_i^{\trn}K x_i + 2\sum_{i=0}^{n-1} \varepsilon_{i}^{\trn}K \wh{x}_i +  \sum_{i=0}^{n-1} \varepsilon^{\trn}_i K \varepsilon_i.
	\]
	Define $Y_n \ldef \sum_{i=0}^n x_i^{\trn}K x_i $. Then this is
        equivalent to
	\begin{align}
		Y_n &\leq \rho Y_n + 2\sum_{i=0}^{n-1} \varepsilon_{i}^{\trn}K \wh{x}_i + \sum_{i=0}^{n-1} \varepsilon^{\trn}_i K \varepsilon_i. \label{eq:recursion}
	\end{align}	
	By assumption we have $x_0 = \zeroes$, and taking an expectation gives us
	$\Ex[Y_n] \leq \rho \Ex[Y_n] + \trace(K)$, which implies that $\Ex[Y_n] \leq \frac{\trace(K)}{(1-\rho)}$.
	Then, using Markov's inequality and the fact that $\sigma_{\min}(K) \geq 1$, we have that with probability at least $1-\delta$,
	\[
	\sum_{i=0}^n x_i^{\trn} x_i \leq \frac{n \trace(K)}{(1-\rho)\delta}.
	\] 
	Furthermore, since $\rho\sum_{i=1}^n x_i^{\trn} K x_i \geq \sum_{i=0}^{n-1} \wh{x}_i^{\trn}K\wh{x}_i$, we have with probability at least $1-\delta$
	\[
	\sum_{i=1}^n \fstar(x_i)^{\trn}K \fstar(x_i) \leq \frac{n \rho \trace(K)}{(1-\rho)\delta}.
	\]
\end{proof}
\begin{lemma}[Strong Upper Bound]
  \label{lem:strong_ub}
  Consider any system \pref{eq:system} for which $\fstar$ is
  $(K,\rho)$-\ges, and let the noise process satisfy
  \pref{ass:noise}. For any $\delta>0$, as soon as $n$ is large enough
  such that $c\tau^2 \sqrt{\frac{d}{n} \log{\prn*{\frac{4\trace(K)}{\delta(1-\rho)}+ 1}}} \leq {\frac{1-\rho}{2}}$
	for some absolute constant $c$, we have that with probability at least
        $1-\delta$,
	\[
	\trace\left(\sum_{i=1}^n x_i K x_i^{\trn}\right) \leq 4n \frac{ \trace(K)}{(1-\rho)}, \quad{}  \quad{} \trace\left(\sum_{i=1}^n x_i x_i^{\trn}\right) \leq 4n \frac{ \trace(K)}{(1-\rho)},
      \]
     and 
	\[
	\lnorm\frac{1}{n}\sum_{i=1}^{n-1} x_i \varepsilon_{i}^{\trn} \rnorm_F \leq 	 c\tau \sqrt{\frac{d}{n}\cdot\frac{\trace(K)}{(1-\rho)}\log{\prn*{\frac{4\trace(K)}{(1-\rho)}\cdot\frac{1}{\delta}+1}}}.
      \]
\end{lemma}
\begin{proof}[\pfref{lem:strong_ub}]
First, observe that
	\begin{align*}
		\sum_{i=1}^n x_i^{\trn} K x_i = \sum_{i=0}^{n-1} \wh{x}_i^{\trn}K\wh{x}_i + 2\sum_{i=0}^{n-1} \varepsilon_{i}^{\trn}K \wh{x}_i + \sum_{i=0}^{n-1} \varepsilon^{\trn}_i K \varepsilon_i,
	\end{align*}
	where $\wh{x}_i \ldef \fstar(x_i)$. As in proof of \pref{lem:weak_ub}, we may use the $(K,\rho)$-\ges property to upper bound by
	\begin{equation}
	\sum_{i=1}^n x_i^{\trn} K x_i \leq \rho\sum_{i=0}^{n-1} x_i^{\trn}K x_i + 2\sum_{i=0}^{n-1} \varepsilon_{i}^{\trn}K \wh{x}_i + \sum_{i=1}^n \varepsilon^{\trn}_i K \varepsilon_i. \label{eq:first_expand}
      \end{equation}
For the remainder of the proof, we define $v$ be an arbitrary random vector with
$\En\brk*{vv^{\trn}}=I$; this random variable is only used for
analysis and is independent of the underlying data generating
process. Observe that for any vector $u$ and matrix $A \succ 0$, we have \[u^{\trn}Au = \trace(A^{1/2}uu^{\trn}A^{1/2}) = \trace(\Ex_v[v^{\trn}
A^{1/2}uu^{\trn}A^{1/2}v]).\]
Since $x_0 = \zeroes$ by
assumption, Eq. \pref{eq:first_expand} becomes
	\begin{align*}
		(1-\rho)\sum_{i=0}^n\Ex_{v}[ v^{\trn}( K^{1/2} x_i x_i^{\trn}K^{1/2})v] &\leq \sum_{i=0}^{n-1} \Ex_{v}[v^{\trn} K^{1/2} \varepsilon_{i} \varepsilon_{i}^{\trn}K^{1/2}v] + 2 \sum_{i=0}^{n-1} \Ex_{v}[v^{\trn} K^{1/2} x_i \varepsilon_{i}^{\trn}K^{1/2}v].
	\end{align*}
	Define random variables $\alpha_{v} \ldef (1/n)\sum_{i=0}^n
        v^{\trn}( K^{1/2} x_i x_i^{\trn}K^{1/2})v$, $ \beta_{v} \ldef
        (1/n) \sum_{i=0}^{n-1} v^{\trn}( K^{1/2} \varepsilon_{i}
        \varepsilon_{i}^{\trn}K^{1/2})v$. Then, we have
	\begin{align}
		(1-\rho) \Ex_v[n \alpha_{v}] &\leq 2 \sum_{i=0}^{n-1} \Ex_{v}[v^{\trn} K^{1/2} \wh{x}_i \varepsilon_{i}^{\trn}K^{1/2}v] + \Ex_v[ n\beta_{v}]. \label{eq:mid_eq}
	\end{align}
        For any $v \in \Rb^d$,  we have $\sum_{i=1}^n v^{\trn} K^{1/2} \wh{x}_i
        \varepsilon_{i}^{\trn}K^{1/2}v \leq \sqrt{n \alpha_{v} + n \beta_{v}}
        \gamma$ by \pref{lem:psd_result_2}, where $\gamma$ is an upper
        bound on
        \[
        \lnorm \left(\sum_{i=0}^n K^{1/2} \wh{x}_i \wh{x}_i^{\trn}K^{1/2} + K^{1/2}\varepsilon_{i} \varepsilon_{i}^{\trn} K^{1/2}\right)^{-1/2}\left(K^{1/2} \wh{x}_i \varepsilon_{i}^{\trn}K^{1/2}\right) \rnorm_{\op} \leq \gamma.
      \]
      We proceed to bound $\gamma$. Since the dynamics satisfy the $(K, \rho)$-\ges property we have
		\[
		\sum_{i=0}^n \wh{x}_i^{\trn}K \wh{x}_i \leq \rho \sum_{i=0}^n x_i^{\trn}K x_i,
		\]
		and by \pref{lem:weak_ub} and \pref{lem:subgaussian_invertible}
        it follows with probability at least $1-\delta$ that
	\begin{equation}
	\sum_{i=0}^n x_i^{\trn}K x_i \leq \frac{n \trace(K)}{(1-\rho)\delta} \rdef \frac{n\wh{B}}{\delta} \quad \text{and} \quad \frac{nI}{2} \preceq \sum_{i=0}^{n-1} \varepsilon_{i} \varepsilon_{i}^{\trn} \preceq \frac{3nI}{2}. \label{eq:upper_eps_x}
	\end{equation}
	By conditioning on \pref{eq:upper_eps_x}, $\sum_{i=0}^n K^{1/2} \wh{x}_i \wh{x}_i^{\trn} K^{1/2} \preceq \frac{n \wh{B}}{\delta}I$ and choosing $V = \frac{3n}{2}K$, we can ensure with probability at least $1-\delta$ that 
	\begin{equation}
	\lnorm \left(\sum_{i=0}^n K^{1/2} (\wh{x}_i \wh{x}_i^{\trn} + \varepsilon_{i} \varepsilon_{i}^{\trn}) K^{1/2}\right)^{-1/2}\left(K^{1/2} \wh{x}_i \varepsilon_{i}^{\trn}K^{1/2}\right) \rnorm_{\op} \leq c\tau \sqrt{d \log{\Big(\frac{\wh{B}}{\delta} + 1\Big)}}. \label{eq:cross_err}
	\end{equation}
	from \pref{lem:selfnorm_ineq}.
	Thus, by setting $\gamma = c\tau \sqrt{d \log{\Big(\frac{\wh{B}}{\delta} + 1\Big)}}$ in \pref{eq:mid_eq} we have with probability at least $1-\delta$
	\begin{align*}
		(1-\rho)\Ex_v[n \alpha_{v}] \leq \Ex_v[\sqrt{n \alpha_{v} + n \beta_{v}}] \gamma + \Ex_v[n \beta_{v}]  \leq \Ex_v[\sqrt{n \alpha_{v}} + \sqrt{n \beta_{v}}] \gamma + \Ex_v[n \beta_{v}].
	\end{align*}
	Whenever $n$ is large enough to satisfy
        $\frac{\gamma}{\sqrt{n}} \leq \frac{1-\rho}{2 \tau}$, this implies
	\[
	\Ex_v[\alpha_{v}] \leq \frac{4\Ex_v[\beta_{v}]}{(1-\rho)}.
	\]
	From \pref{lem:subgaussian_invertible}, we have 
	 \begin{equation}
	 \frac{1}{n}\sum_{i=0}^{n-1} \varepsilon_{i} \varepsilon_{i}^{\trn} \preceq \frac{3I}{2} \label{eq:err_isometry}
	 \end{equation}
	 with probability $1-\delta$ whenever $n\geq{}c\tau^{4}(d+\log(2/\delta))$, which in particular is satisfied when $\frac{\gamma}{\sqrt{n}} \leq \frac{1-\rho}{2 \tau}$. Then, conditioning on both \pref{eq:cross_err} and \pref{eq:err_isometry}, we have $\Ex_v[\beta_{v}] \leq \frac{3}{2}\Ex_v[v^{\trn} K v]$ and 
	\[
	\trace\prn*{\sum_{i=1}^n x_i K x_i^{\trn}} \leq \frac{4n\trace(K)}{(1-\rho)}.
	\]
	The main claim now follows because \pref{eq:cross_err} and \pref{eq:err_isometry} occur simultaneously with probability at least $1-2\delta$. Furthermore, we also have 
	\begin{align*}
		2  \lnorm \frac{1}{n}\sum_{i=1}^{n-1} K^{1/2}x_i \varepsilon_{i}^{\trn} K^{1/2}\rnorm_F &= \sqrt{\Ex_{v}\lbr \lnorm\frac{1}{n}\sum_{i=1}^{n-1}K^{1/2} x_i \varepsilon_{i}^{\trn}K^{1/2}v \rnorm_2^2\rbr }\leq 
		2\frac{\gamma}{\sqrt{n}} \sqrt{\Ex_{v}[(\alpha_{v} + \beta_{v})] }   
	\end{align*}
with probability atleast $1 -2\delta$. Simplifying, this implies
\[
  \lnorm \frac{1}{n}\sum_{i=1}^{n-1} K^{1/2}x_i \varepsilon_{i}^{\trn} K^{1/2}\rnorm_F \leq  c\tau \sqrt{\frac{d\trace{(K)}}{n(1-\rho)}\log{\prn*{\frac{4\trace(K)}{\delta(1-\rho)}+1}}},
\]
for some absolute constant $c>0$. Finally, the condition on $n$ in the
lemma statement follows by simplifying and expanding the condition $\frac{\gamma}{\sqrt{n}} \leq \frac{1-\rho}{2 \tau}$.
\end{proof}
\begin{proof}[\pfref{thm:ges_stable}]
The lower isometry bound is attained by applying \pref{lem:sample_cov} using the upper bound from \pref{lem:weak_ub}, and the upper isometry bound follows immediately from \pref{lem:strong_ub}.
\end{proof}


	\section{Proofs from \pref{sec:learning}}
	For the remainder of the appendix we make use of the 
	filtration
	\begin{equation}
	\cG_i \ldef{} \sigma(\veps_0,x_1,\veps_1,\ldots,x_{i-1},\veps_{i-1},x_i).\label{eq:filtration}
	\end{equation}

	\subsection{Proof of \pref{thm:known_link_main}}
	\label{app:known_link}

We first define two parameters:
	\begin{equation}
	\err = c\tau \sqrt{\frac{d\trace{(K)}}{n(1-\rho)}\log{\left(\frac{4\trace(K)}{\delta(1-\rho)}+1\right)}}, \quad{}\text{and}\quad{} B = \frac{ 4\trace(K)}{(1-\rho)}.\label{eq:mu_def}
	\end{equation}
Let us proceed with the proof. To begin, recall from \pref{lem:strong_ub} that whenever
	\[
	c \tau \sqrt{\frac{d}{n} \log{\left(\frac{B}{\delta}+ 1\right)}} \leq {\frac{1-\rho}{2 \tau}},
	\]
	we have with probability at least $1-\delta$,
	\[
	\trace\left(\frac{1}{n}\sum_{i=1}^n x_i x_i^{\trn}\right) \leq \frac{ 4\trace(K)}{(1-\rho)}, \quad{}\text{and}\quad{} \lnorm \frac{1}{n} \sum_{i=1}^n \varepsilon_{i} x_i^{\trn} \rnorm_F \leq c\tau \sqrt{\frac{d\trace{K}}{n(1-\rho)}\log{\left(\frac{4\trace(K)}{\delta(1-\rho)}+1\right)}}.
      \]
      From here, we condition on this good event above and split the
      proof into two cases. First, we handle only the prediction error/denoising
      guarantee \pref{eq:thm2_nondecreasing}, which requires only that
      the link function $\sigma$ is non-decreasing and does not
      require a lower bound on the link derivative. Then, in the
      second part of the proof, we use the assumption that the
      $\sigma$ is strictly increasing to strengthen this bound and provide a
      parameter recovery guarantee. The first part of the proof
      extends the arguments of \cite{kakade2011efficient} to the dependent setting
      where data is generated from a generalized linear system, while
      the second part uses the refined isometry guarantees developed
      in \pref{sec:stability}.

\subsubsection{\pfref{thm:known_link_main}, Part I: Slow rate for
  prediction error}
\label{app:strict_non_monotone}
Throughout this proof we use that the projection operation in
\pref{alg:parameter_estimates} ensures that for all $t$,
$\lnorm\Theta^{(t)} - \Thetas\rnorm_F \leq 2W$. From \pref{alg:parameter_estimates} we have that
	\begin{align}
	\lnorm\Thetatt - \Thetas \rnorm_F &= \lnorm \text{Proj}_{\Mc}(\Thetat - \lr \hG(\Thetat, X)) -\Thetas\rnorm_F, \nonumber 
	\end{align}
which implies that
	\begin{align}
	\lnorm\Thetatt - \Thetas\rnorm_F &\leq \lnorm\Thetat - \lr \hG(\Thetat, X) -\Thetas\rnorm_F. \nonumber 
	\end{align}
	 Furthermore, recall that $\hG(\Thetat, X)$ is defined as
         $\hG(\Theta^{(t)}, X) = \frac{1}{n} \sum_{i=1}^n
         (\sigma(\Theta^{(t)} x_{i}) - x_{i+1})x_i^{\trn}$. Thus, by
         expanding the right-hand side above, we have
	\begin{align}
	\lnorm \Thetat - \Thetas\rnorm^2_F - \lnorm\Thetatt - \Thetas\rnorm^2_F &\geq \frac{2 \lr}{n} \sum_{i=1}^n \la x_{i+1} - \sigma(\Thetat x_i), (\Thetas - \Thetat)x_i \ra \label{eq:first_eq1} \\
	&~~~~- \lr^2 \lnorm \frac{1}{n}\sum_{i=1}^n (x_{i+1} - \sigma(\Thetat x_i))x_i^{\trn} \rnorm_F^2. \label{eq:first_eq2}
	\end{align}
	To obtain a lower bound on $	\lnorm\Thetat -
        \Thetas\rnorm^2_F - \lnorm\Thetatt - \Thetas\rnorm^2_F$ we
        need a lower bound on the right-hand side in
        \pref{eq:first_eq1} and upper bound on the norm in \pref{eq:first_eq2}. Analyzing \pref{eq:first_eq1} first, we get 
	\begin{align*}
	\frac{2}{n} \sum_{i=1}^n \la x_{i+1} - \sigma(\Thetat x_i), (\Thetas - \Thetat)x_i \ra &= \frac{2}{n} \sum_{i=1}^n \la \sigma(\Thetas x_i) - \sigma(\Thetat x_i), (\Thetas - \Thetat)x_i \ra \\
	&~~~~+ \frac{2}{n} \sum_{i=1}^n \la \varepsilon_{i}, (\Thetas - \Thetat)x_i \ra.
	\end{align*}
	Note that since $\lnorm \frac{2}{n} \sum_{i=1}^n
        \varepsilon_{i} x_i^{\trn} \rnorm_F \leq 2\err$, and
        $\lnorm\Thetas - \Thetat \rnorm_F \leq 2W$, we have
	\begin{equation}
	\frac{2}{n} \sum_{i=1}^n \la \varepsilon_{i}, (\Thetas - \Thetat)x_i \ra \geq - \lnorm \frac{2}{n} \sum_{i=1}^n \varepsilon_{i} x_i^{\trn} \rnorm_F \lnorm\Thetas - \Thetat\rnorm_F \geq -4W \err. \label{eq:self_norm_cross}
	\end{equation}
Since each $\sigma_i$ is non-decreasing and $1$-Lipschitz, we also have
	\begin{equation}
	\frac{2}{n} \sum_{i=1}^n \la \sigma(\Thetas x_i) - \sigma(\Thetat x_i), (\Thetas - \Thetat)x_i \ra \geq \frac{2}{n} \sum_{i=1}^n \lnorm \sigma(\Thetas x_i) - \sigma(\Thetat x_i)\rnorm_2^2 . \label{eq:inner_lower_bnd}
	\end{equation}
	Together, these lead to the following lower bound on \pref{eq:first_eq1}:
	\[
	\frac{2 \lr}{n} \sum_{i=1}^n \la x_{i+1} - \sigma(\Thetat x_i), (\Thetas - \Thetat)x_i \ra \geq 2 \lr( \cE({\Theta}^{(t)}) - 2 W \err).
	\]
	For \pref{eq:first_eq2} we have 
	\begin{align*}
	\lnorm\frac{1}{n}\sum_{i=1}^n (x_{i+1} - \sigma(\Thetat x_i))x_i^{\trn}\rnorm_F^2 &= 	\lnorm\frac{1}{n}\sum_{i=1}^n (x_{i+1} - \sigma(\Thetas x_i) + \sigma(\Thetas x_i) - \sigma(\Thetat x_i))x_i^{\trn}\rnorm_F^2 \\
	&\leq 2\lnorm\frac{1}{n}\sum_{i=1}^n (x_{i+1} - \sigma(\Thetas x_i))x_i^{\trn}\rnorm_F^2 + 2\lnorm\frac{1}{n}\sum_{i=1}^n (\sigma(\Thetat x_i) - \sigma(\Thetas x_i))x_i^{\trn}\rnorm_F^2
	\end{align*}
	Since $\veps_i = x_{i+1} - \sigma(\Thetas x_i)$ we have
	\[
	\lnorm \frac{1}{n}\sum_{i=1}^n (x_{i+1} - \sigma(\Thetas x_i))x_i^{\trn}\rnorm_F^2 \leq \err^2, 
	\]
	and by the Cauchy Schwarz inequality
	\[
	\lnorm \frac{1}{n}\sum_{i=1}^n (\sigma(\Thetat x_i) - \sigma(\Thetas x_i))x_i^{\trn}\rnorm_F^2 \leq\left(\frac{1}{n}\sum_{i=1}^n \lnorm  (\sigma(\Thetat x_i) - \sigma(\Thetas x_i)) \rnorm^2 \right) \cdot\trace\left(\frac{1}{n}\sum_{i=1}^n x_i x_i^{\trn}\right).
	\]
        Combining these upper bounds, we have
	\[
		\lnorm \frac{1}{n}\sum_{i=1}^n (x_{i+1} - \sigma(\Thetat x_i))x_i^{\trn}\rnorm_F^2 \leq 2\err^2 + 2\left(\frac{1}{n}\sum_{i=1}^n \lnorm  (\sigma(\Thetat x_i) - \sigma(\Thetas x_i)) \rnorm^2 \right) \cdot\trace\prn*{\frac{1}{n}\sum_{i=1}^n x_i x_i^{\trn}}.
	\]
	Since $\trace(\frac{1}{n}\sum_{i=1}^n x_i x_i^{\trn}) \leq B$,
        we can further upper bound the right-hand side using
	\[
	\lnorm\frac{1}{n}\sum_{i=1}^n (\sigma(\Thetat x_i) - \sigma(\Thetas x_i))x_i^{\trn}\rnorm_F^2 \leq \cE({\Theta}^{(t)})\cdot{}B.
	\]
	Combining the bounds for \pref{eq:first_eq1} and \pref{eq:first_eq2} we have
	\[
	\lnorm\Thetat - \Thetas\rnorm^2_F - \lnorm\Thetatt - \Thetas\rnorm^2_F \geq 2 \lr( \cE({\Theta}^{(t)}) - 2 W \err) - 2 \lr^2 (\err^2 + \cE({\Theta}^{(t)}) B ).
	\]
	By choosing $\lr = \frac{1}{4B}$ we get 
	\begin{equation}
	\label{eq:thet_diff}
	\lnorm\Thetat - \Thetas\rnorm^2_F - \lnorm\Thetatt - \Thetas\rnorm^2_F \geq c\left(\frac{\cE({\Theta}^{(t)})}{B} -\frac{5 W \err}{B} -  \frac{W \err^2}{B^2} \right),
	\end{equation}
	Summing Eq.~\pref{eq:thet_diff} we have
	\begin{align}
          \frac{1}{m}\cdot{}\sum_{t=0}^{m-1}\prn*{	\lnorm\Thetat - \Thetas\rnorm^2_F - \lnorm\Thetatt - \Thetas\rnorm^2_F} &\geq{} c\left(\frac{\sum_{t=0}^{m-1}\cE({\Theta}^{(t)})}{mB} -\frac{5 W \err}{B} -  \frac{W \err^2}{B^2} \right). \nonumber
	\end{align}
        Observing that the left-hand side telescopes, after
        rearranging we have
	\begin{align}
	\Ex_{\cA}[\cE(\wh{\Theta})] &\leq 5 W \err +  \frac{W \err^2}{B} + \frac{BW^2}{cm} \label{eq:inter}
	\end{align}
	where $\Ex_{\cA}[\cE(\wh{\Theta})] = \frac{\sum_{t=0}^{m-1}\cE({\Theta}^{(t)})}{mB}$. We choose the number of iterations, $m$, such that $m \geq \frac{WB}{5 \mu}$ and then \pref{eq:inter} becomes
	\[
          \Ex_{\cA}[\cE(\wh{\Theta})] \leq 10 W \err +  \frac{W \err^2}{B} \leq 20 (W/B) \cdot{} \prn*{B \err \vee \err^2}.
	\]
        \qed
	
\subsubsection{\pfref{thm:known_link_main}, Part II: Fast rate for
  prediction and parameter recovery}
\label{app:strict_monotone}
Compared to the slow rate setting, to attain fast rates for strictly
increasing link functions, we provide a tighter bound on
the term $\frac{2}{n} \sum_{i=1}^n \la \varepsilon_{i}, (\Thetas -
\Thetat)x_i \ra$ in \pref{eq:self_norm_cross} so that only terms
proportional to $\err^2$ remain. This is made possible by the fact
that we can lower bound the prediction error in terms of $\lnorm\Thetat - \Thetas\rnorm$.
		
To begin, since $|\sigma(a) - \sigma(b)| \geq \zeta |a -b| > 0$, we have
	\[
	 \frac{1}{n}\sum_{i=1}^n \lnorm (\sigma(\Thetat x_i) - \sigma(\Thetas x_i)) \rnorm^2 \geq  \frac{\zeta^2}{n}\sum_{i=1}^n \lnorm \Thetat x_i - \Thetas x_i \rnorm^2.
	\] 
	Recall from \pref{eq:first_eq1} and \pref{eq:first_eq2} that
	\begin{align}
	\lnorm\Thetat - \Thetas\rnorm^2_F - \lnorm\Thetatt - \Thetas\rnorm^2_F &\geq \frac{2 \lr}{n} \sum_{i=1}^n \la x_{i+1} - \sigma(\Thetat x_i), (\Thetas - \Thetat)x_i \ra \label{eq:d1}\\
	&~~~~- \lr^2 \lnorm \frac{1}{n}\sum_{i=1}^n (x_{i+1} - \sigma(\Thetat x_i))x_i^{\trn} \rnorm_F^2. \label{eq:d2}
	\end{align}
Using the same analysis as in \pref{app:strict_non_monotone}, we have	
	\begin{equation}
	\lr^2 \lnorm \frac{1}{n}\sum_{i=1}^n (x_{i+1} - \sigma(\Thetat x_i))x_i^{\trn} \rnorm_F^2 \leq 2 \lr^2 (\err^2 + \cE({\Theta}^{(t)}) B ) \leq 2 \lr^2 \left(\err^2 +  \lnorm\Thetat - \Thetas\rnorm^2_F B^2  \right), \label{eq:inter_1}
	\end{equation}
	since $\cE(\Thetat) \leq  \lnorm\Thetat - \Thetas\rnorm^2_F
        B$. Next, recall that
	\begin{align}
		\frac{2}{n} \sum_{i=1}^n \la x_{i+1} - \sigma(\Thetat x_i), (\Thetas - \Thetat)x_i \ra &= \frac{2}{n} \sum_{i=1}^n \la \sigma(\Thetas x_i) - \sigma(\Thetat x_i), (\Thetas - \Thetat)x_i \ra \notag\\
		&~~~~+ \frac{2}{n} \sum_{i=1}^n \la \varepsilon_{i}, (\Thetas - \Thetat)x_i \ra, \notag\\
		&\geq \frac{2}{n} \sum_{i=1}^n \lnorm\sigma(\Thetas x_i) - \sigma(\Thetat x_i)\rnorm^2 + \frac{2}{n} \sum_{i=1}^n \la \varepsilon_{i}, (\Thetas - \Thetat)x_i \ra,\notag
	\end{align}
	where the inequality follows because $\sigma(\cdot)$ is
        coordinate wise $1$--Lipschitz and
        non--decreasing. Furthermore, by \pref{lem:sample_cov} we have
        that once $n$ is sufficiently large (in particular, when $n$
        satisfies the assumptions of the theorem statement), with high
        probability, $\sum_{i=1}^n x_i x_i^{\trn} \succeq
        \frac{nI}{4}$. This in turn implies that
	\[
	 \frac{1}{n}\sum_{i=1}^n \lnorm (\sigma(\Thetat x_i) - \sigma(\Thetas x_i)) \rnorm^2 \geq  \frac{\zeta^2}{n}\sum_{i=1}^n \lnorm \Thetat x_i - \Thetas x_i \rnorm^2 \geq \frac{\zeta^2}{4} \lnorm \Thetat  - \Thetas  \rnorm^2_F.
	\]
	Next, introducing a free parameter $\gamma>0$, we lower bound
        $\frac{2}{n} \sum_{i=1}^n \la \varepsilon_{i}, (\Thetas -
        \Thetat)x_i \ra$ via Cauchy-Schwarz and the AM-GM inequality
	\begin{align}
	\frac{2}{n} \sum_{i=1}^n \la \varepsilon_{i}, (\Thetas - \Thetat)x_i \ra &\geq - \lnorm\frac{2}{n} \sum_{i=1}^n \varepsilon_{i} x_i^{\trn} \rnorm_F \lnorm\Thetas - \Thetat\rnorm_F \geq - \gamma \lnorm\frac{2}{n} \sum_{i=1}^n \varepsilon_{i} x_i^{\trn} \rnorm^2_F - \frac{1}{\gamma}\lnorm\Thetas - \Thetat\rnorm_F^2. \label{eq:am_gm_fast}
	\end{align}
	Since $\lnorm\frac{2}{n} \sum_{i=1}^n \varepsilon_{i}
        x_i^{\trn} \rnorm^2_F \leq 4 \err^2$, we have
	\[
	\frac{2}{n} \sum_{i=1}^n \la \varepsilon_{i}, (\Thetas - \Thetat)x_i \ra \geq - 4\gamma \err^2  -\frac{1}{\gamma}\lnorm\Thetas - \Thetat\rnorm_F^2.
	\]
	With these developments, we deduce the following lower bound
        on the right-hand side of \pref{eq:d1}:
	\begin{equation}
\frac{2}{n} \sum_{i=1}^n \la x_{i+1} - \sigma(\Thetat x_i), (\Thetas - \Thetat)x_i \ra \geq \frac{\zeta^2}{4} \lnorm \Thetat  - \Thetas  \rnorm^2_F - \frac{1}{\gamma}\lnorm \Thetat  - \Thetas  \rnorm^2_F - 4\gamma \err^2.
	\end{equation}
	By setting $\gamma = \frac{8}{\zeta^2}$, this becomes
	\begin{equation}
          \label{eq:first_eq_12}
	\frac{2}{n} \sum_{i=1}^n \la x_{i+1} - \sigma(\Thetat x_i), (\Thetas - \Thetat)x_i \ra \geq \frac{\zeta^2}{8} \lnorm \Thetat  - \Thetas  \rnorm^2_F  - \frac{32}{\zeta^2} \err^2.
	\end{equation}
      Combining \pref{eq:d1}, \pref{eq:first_eq_12}, and
      \pref{eq:inter_1}, we have
	\[
\lnorm\Thetat - \Thetas\rnorm^2_F - \lnorm\Thetatt - \Thetas\rnorm^2_F \geq 2 \lr \left( \frac{\zeta^2}{8} \lnorm \Thetat  - \Thetas  \rnorm^2_F  - \frac{32}{\zeta^2} \err^2 \right) - 2 \lr^2 (\err^2 + \lnorm \Thetat  - \Thetas  \rnorm^2_F B^2 ).
\]	
By selecting $\lr = \frac{\zeta^2}{16B^2}$, we have that for absolute constants $c_1, c_2$,
\[
\lnorm\Thetat - \Thetas\rnorm^2_F - \lnorm\Thetatt - \Thetas\rnorm^2_F \geq \frac{c_1 \zeta^4 \lnorm\Thetat - \Thetas\rnorm^2_F }{B^2} -  \frac{c_2 \err^2  }{B^2}.
\]
Rearranging, we have
\begin{align*}
	\nrm{\Thetatt - \Thetas}_F^2 &\leq \left(1 - \frac{c_1 \zeta^4}{B^2}\right) \nrm{\Thetat - \Thetas}_F^2  + \frac{c_2 \err^2}{B^2},
\end{align*}
and by applying this bound recursively, we get
\begin{align*}
	\nrm{\Thetatt - \Thetas}_F^2 &\leq \left(1 - \frac{c_1 \zeta^4}{B^2}\right)^t W^2 + \frac{c_2 \err^2}{c_1\zeta^4} .
\end{align*}
Consequently, there exists an absolute constant $c$ such that for $t \geq \frac{c B^2}{\zeta^4} \log{\left(\frac{W^2\zeta^4}{\err^2}\right)} \vee 1$, we have 
\[
\nrm{\Thetatt - \Thetas}_F^2 \leq  \frac{2c_2\err^2}{c_1\zeta^4}.
\]

We can now use the bound on $\nrm{\Thetatt - \Thetas}_F^2$, to obtain an upper bound on
\[
\cE(\Theta^{(t)}) = \frac{1}{n}\sum_{i=1}^n \lnorm (\sigma(\Thetat x_i) - \sigma(\Thetas x_i)) \rnorm^2.
\] 
We continue from the following inequality:
\begin{align*}
\lnorm\Thetat - \Thetas\rnorm^2_F - \lnorm\Thetatt - \Thetas\rnorm^2_F &\geq{} 2 \lr \left( \cE(\Theta^{(t)})  - 4\gamma \err^2  -\frac{1}{\gamma}\lnorm\Thetas - \Thetat\rnorm_F^2 \right)\\ 
&-{} 2 \lr^2 (\err^2 + \lnorm \Thetat  - \Thetas  \rnorm^2_F B^2 ).
\end{align*}
By choosing $t = m$, we get 
\[
 2 \eta_m \cE(\Theta^{(m)}) \leq{} \prn*{1 + \frac{2 \eta_m}{\gamma} + 2 \eta_m^2B^2} \cdot{} \lnorm\Theta^{(m)} - \Thetas\rnorm^2_F + \prn*{8\eta_m \gamma + 2 \eta_m^2} \cdot{} \err^2, 
\]
from which we conclude that 
\[
 \cE(\Theta^{(m)})  \leq{} c\prn*{\frac{1}{\eta_m} + \gamma} \cdot{} \mu^2 .
\]
\qed


        \subsection{\pfref{thm:relu_case}}
	\label{app:learn_relu}
	
Throughout this section of the appendix we overload notation and use
$\sigma(x)=\relu(x)$ for $x\in\bbR$ and $\sigma(x) =
(\relu(x_1),\ldots,\relu(x_d))$ for $x\in\bbR^{d}$.

\subsubsection{Structural results for ReLU prediction error}
\begin{lemma}[\cite{cho2009kernel}]
	\label{lem:gaussian_product}
	Let $\veps \sim \Nc(0, I)$. For vectors $u$ and $v$ denote by $\theta_{u, v}$ the angle between $u$ and $v$. Assume $\theta_{u, v} \in [0, \pi]$, then:
	\[
	\Ex[\sigma(\la u, \veps \ra) \cdot \sigma(\la v, \veps \ra)] = \frac{1}{2\pi}\nrm{u}\nrm{v} \left(\sin{\theta_{u, v}} + (\pi - \theta_{u, v}) \cos{\theta_{u, v}}\right).
	\]	
      \end{lemma}
  \begin{proof}
  	From Eq. (1) in~\cite{cho2009kernel}, we have
  	\[
  		\Ex[\sigma(\la u, \veps \ra) \cdot \sigma(\la v, \veps \ra)] = k_1(u, v),
  	\] 
  	where $k_1(u,v)$ is defined in Eq. (3) in~\cite{cho2009kernel} as
  	\[
  	k_1(u, v) = \frac{1}{\pi} \cdot{}\nrm{u} \nrm{v} J_1(\theta_{u,v}),
  	\]
  	with $J_1(\theta) \ldef{}\sin{\theta} + (\pi -
        \theta)\cos{\theta}$ (see \cite{cho2009kernel}, Eq. (6)).
  \end{proof}
\begin{proposition}
	\label{prop:trig}
	Let $u, v \in \Rb^d$, and let $\theta_{u, v}$ be the angle
        between $u$ and $v$. Suppose that $\theta_{u, v} \in [0,
        \pi]$. Then
	\[
\frac{2\theta_{u, v}}{\pi}\nrm{u}\nrm{v}\sin^2{\prn*{\theta_{u, v}/2}} \leq \frac{1}{4}\cdot{}\nrm{u-v}^2.
	\]
\end{proposition}
\begin{proof}[\pfref{prop:trig}]
We first expand the square as
\begin{align*}
\frac{1}{4}\cdot{}\nrm{u-v}^2 &= \frac{1}{4} \cdot{} \prn*{\nrm{u}^2 +
                                \nrm{v}^2 - 2\nrm{u} \nrm{v}
                                \cos{\theta_{u, v}}}.
                                \intertext{Next, using the AM-GM
                                inequality, we lower bound by}
 &\geq \frac{1}{2}\cdot{} \prn*{\nrm{u} \nrm{v} - \nrm{u} \nrm{v}
   \cos{\theta_{u, v}}}.
   \intertext{Using the half-angle identity, this is equal to}
  &=2 \nrm{u} \nrm{v} \sin^2{\prn*{\theta_{u, v}/2}}.\\
  \intertext{Finally, we use that $\theta_{u,v}\in\brk*{0,\pi}$:}
  &\geq \frac{2\theta_{u, v}}{\pi}\nrm{u}\nrm{v}\sin^2{\prn*{\theta_{u, v}/2}}.
\end{align*}
\end{proof}
      
\begin{lemma}
  \label{lem:gaussian_relu}
  Suppose that $\veps \sim \Nc(0, \gamma I)$ for some $\gamma >
  0$. Then for any two vectors $u$ and $v$,
	\[
	\Ex[(\sigma(\la u, \veps\ra) - \sigma(\la v, \veps\ra))^2] \geq \frac{\gamma}{4} \|u-v\|^2.
	\]
\end{lemma}
\begin{proof}[\pfref{lem:gaussian_relu}]
  Observe that if $\veps'\sim \Nc(0, I)$ we have
	\[
	\Ex[(\sigma(\la u, \veps\ra) - \sigma(\la v, \veps\ra))^2] = \gamma \cdot \Ex[(\sigma(\la u, \veps'\ra) - \sigma(\la v, \veps'\ra))^2].
      \]
      Thus, going forward we assume $\veps\sim{}\cN(0,I)$ without loss
      of generality. First, observe that
	\[
	\Ex[(\sigma(\la u, \veps\ra) - \sigma(\la v, \veps\ra))^2] = \Ex[( \sigma(\la v, \veps\ra))^2] + \Ex[(\sigma(\la u, \veps\ra))^2] - 2\Ex[\sigma(\la u, \veps\ra) \cdot \sigma(\la v, \veps\ra)].
	\]
	Without loss of generality, we will assume that $\theta_{u, v}
        \in [0, \pi]$, where $\theta_{u,v}$ is the angle between $u$
        and $v$. Using \pref{lem:gaussian_product}, we have the identity
	\[
	\Ex[\sigma(\la u, \veps \ra) \cdot \sigma(\la v, \veps \ra)] = \frac{1}{2\pi}\nrm{u}\nrm{v} \left(\sin{\theta_{u, v}} + (\pi - \theta_{u, v}) \cos{\theta_{u, v}}\right).
	\]
We now lower bound the right-hand side as
	\begin{align*}
          \Ex[(\sigma(\la u, \veps\ra) - \sigma(\la v, \veps\ra))^2] &= \frac{1}{2}\nrm{u}^2 + \frac{1}{2} \nrm{v}^2 -  \frac{1}{\pi}\nrm{u}\nrm{v} \left(\sin{\theta_{u, v}} + (\pi - \theta_{u, v}) \cos{\theta_{u, v}}\right), \\
                                                                     &\overset{(i)}{=} \frac{1}{2}\nrm{u-v}^2 - \frac{1}{\pi}\nrm{u}\nrm{v}(\sin{\theta_{u,v}} - \theta_{u,v}\cos{\theta_{u,v}}),\\
                                                                     &= \frac{1}{2}\nrm{u-v}^2 - \frac{\theta_{u, v}}{\pi}\nrm{u}\nrm{v}\prn*{\frac{\sin{\theta_{u,v}}}{\theta_{u,v}} - \cos{\theta_{u,v}}},\\
                                                                     &\geq	\frac{1}{2}\nrm{u-v}^2 - \frac{\theta_{u,
           v}}{\pi}\nrm{u}\nrm{v}\prn*{1 - \cos{\theta_{u,v}}} \\
                                                                     &\overset{(ii)}{=}	\frac{1}{2}\nrm{u-v}^2 - \frac{2\theta_{u, v}}{\pi}\nrm{u}\nrm{v}\sin^2{\prn*{\theta_{u, v}/2}},
	\end{align*}
        where $(i)$ uses that
        $\tri*{u,v}=\nrm*{u}\nrm*{v}\cos\theta_{u,v}$ and $(ii)$ uses
        the half-angle identity
        $\sin^{2}(\theta/2)=\frac{1-\cos{}\theta}{2}$. Now from
        \pref{prop:trig}, we have
\begin{equation}
\label{eq:lb}
\frac{2\theta_{u, v}}{\pi}\nrm{u}\nrm{v}\sin^2{\prn*{\theta_{u, v}/2}} \leq \frac{1}{4}\nrm{u-v}^2.
\end{equation}
It follows that
	\begin{align*}
\Ex[(\sigma(\la u, \veps\ra) - \sigma(\la v, \veps\ra))^2] &\geq \frac{1}{4}\nrm{u-v}^2.
\end{align*}
\end{proof}
Recall that in the statement of \pref{thm:relu_case} we assume that $\{\veps_i\}_{i=1}^n$ are i.i.d. Gaussian with $\veps_i \sim \Nc(0, I)$. Since
\[
x_{i+1} = \sigma(\Thetas x_i) + \veps_{i},
\]
this implies that $x_{i+1}$ is distributed as $\Nc\left(\sigma(\Thetas
  x_i), I \right)$ conditioned on $\cG_i$. This leads to the following result.
\begin{lemma}
	\label{lem:expectation_prop}
	Define $\mu_i = \sigma(\Thetas x_i)$ and let $c$ be an absolute constant. For any two vectors $u, v$, we have 
	\[
	\Ex[(\sigma(\la u, x_{i+1} \ra) - \sigma(\la v, x_{i+1} \ra))^2\mid\cG_{i}] \geq \frac{1}{4}\cdot{} e^{-\|\mu_i\|^2}\|u-v\|^2.
	\]
      \end{lemma}
\begin{proof}[\pfref{lem:expectation_prop}]
	 Let $P_{u, v}(x)$ denote the orthogonal projection of $x$ onto the plane spanned by $u$ and $v$, and let $\wh{u}=P_{u,v}(u)$ and $\wh{v}=P_{u,v}(v)$. Then
	\[
	(\sigma(\la \wh{u}, x_{i+1} \ra) - \sigma(\la \wh{v}, x_{i+1} \ra))^2 = (\sigma(\la \wh{u}, P_{u,v}(x_{i+1}) \ra) - \sigma(\la \wh{v}, P_{u,v}(x_{i+1}) \ra))^2.
	\]
Observe that the distribution of $P_{u,v}(x)$ under $x\sim{}\cN(\mu_i,I_{d\times{}d})$, is equivalent in law to $\Nc(P_{u,v}(\mu_i), I_{2\times{}2})$.
Consequently, we have
	\begin{align*}
		\Ex[(\sigma(\la u, x_{i+1} \ra) - \sigma(\la v, x_{i+1} \ra))^2\mid\cG_{i}]	&= \frac{1}{2 \pi} \cdot{} \int(\sigma(\la \wh{u}, y \ra) - \sigma(\la \wh{v}, y \ra))^2 e^{\frac{-\|y - P_{u,v}(\mu_i)\|^2}{2}}dy, \\
		&\geq \frac{1}{2 \pi} \cdot{} \int(\sigma(\la \wh{u}, y \ra) - \sigma(\la \wh{v}, y \ra))^2 e^{-(\|y\|^2 + \|\mu_i\|^2)}dy,\\
		&= \frac{e^{-\|\mu_i\|^2}}{2} \cdot{}\frac{1}{\pi} \int(\sigma(\la \wh{u}, y \ra) - \sigma(\la \wh{v}, y \ra))^2 e^{-\|y\|^2}dy.
	\end{align*}
        Define a gaussian vector $\veps  \sim \Nc\prn*{0,  \frac{1}{2}I_{2\times{}2}}$.
	Observe that $\Ex[(\sigma(\la \wh{u}, \veps \ra) -
        \sigma(\la \wh{v}, \veps
        \ra))^2]=(1/\pi) \cdot{} \int(\sigma(\la
        \wh{u}, y \ra) - \sigma(\la \wh{v}, y \ra))^2 e^{-\|y\|^2 }
        dy$. Furthermore, from \pref{lem:gaussian_relu} we have:
	\[
          \Ex[(\sigma(\la \wh{u}, \veps \ra) - \sigma(\la \wh{v}, \veps \ra))^2] \geq \frac{1}{4} \cdot{} \|\wh{u}-\wh{v}\|^2=\frac{1}{4}\|u-v\|^2.
	\]
	Together, these inequalities imply that
	\[
	\Ex[(\sigma(\la u, x_{i+1} \ra) - \sigma(\la v, x_{i+1} \ra))^2\mid\cG_{i}] \geq \frac{1}{4} \cdot{} e^{- \|\mu_i\|^2} \|u-v\|^2.
	\]
\end{proof}

\subsubsection{Relating prediction error to parameter error}
To prove parameter convergence for the ReLU we first establish the
following key lemma, which states that with high probability a certain
variant of the prediction error for $\Theta$ is lower bounded by the
parameter recovery error.
\begin{lemma}
  \label{lem:relu_lem}
  Let the assumptions of \pref{thm:relu_case} hold. Then whenever
	\[
	n \geq \frac{c\tau^4d^3}{1-\rho }\log{\left(1 + \frac{\trace(K)}{\delta(1-\rho )}\right)} ,
	\]
	we have that with probability at least $1-\delta$, for all $\Theta^{(1)}, \Theta^{(2)} \in \Mc$ simultaneously,
	\[
	\frac{1}{n}\sum_{i=1}^n \Ex\left[\lnorm \sigma(\Theta\ind{1} x_i) - \sigma(\Theta\ind{2} x_i)\rnorm^2\mid\cG_{i-1}\right] \geq (1/4) \cdot{} e^{\frac{-4\rho \trace(K)}{1-\rho }}\|\Theta\ind{1}- \Theta\ind{2} \|_F^2 - \frac{1}{n}.
	\]
\end{lemma}
\begin{proof}
	Fix $\Theta^{(1)}, \Theta^{(2)} \in \Mc$. Let $u_j$ and $v_j$
        denote the $j^{\text{th}}$ rows of $\Theta^{(1)}$ and $\Theta^{(2)}$ respectively. Then 
	\[
	\frac{1}{n} \sum_{i=1}^n \Ex\left[\lnorm \sigma(\Theta^{(1)} x_i) - \sigma(\Theta^{(2)} x_i)\rnorm^2\mid\cG_{i-1}\right] = \frac{1}{n}\sum_{j=1}^d\sum_{i=1}^n \Ex[(\sigma(\la u_j, x_{i} \ra) - \sigma(\la v_j, x_{i} \ra))^2\mid\cG_{i-1}].
	\]
	Note that $\sum_{i=1}^n \|\mu_i\|^2 = \sum_{i=1}^n \|\sigma(\Thetas x_i)\|^2 \leq \sum_{i=1}^n \|\Thetas x_i\|^2$. Since $(\Thetas)^{\trn}K \Thetas \preceq \rho  K$ we have 
	\[
	x_i^{\trn} (\Thetas)^{\trn}K\Thetas x_i \leq \rho  x_i^{\trn} K x_i .
	\]
	Since $K \succeq I$, this implies that $\sum_{i=1}^n \|\Thetas
        x_i\|^2 \leq \rho  \sum_{i=1}^n
        \trace\left(x_i^{\trn}Kx_i\right)$, and consequently
        \pref{lem:strong_ub} guarantees that once  $n \geq \frac{c\tau^4d}{1-\rho }\log{\left(1 +
		\frac{\trace(K)}{\delta(1-\rho )}\right)}$, with probability at least $1-\delta$
	\[
	\sum_{i=1}^n \|\mu_i\|^2 \leq \frac{4n\rho \trace\left(K\right)}{1-\rho }.
	\]
	Now, we know from \pref{lem:expectation_prop} that
        \begin{align*}
	\frac{1}{n}\sum_{i=1}^n \Ex[(\sigma(\la u_j, x_{i} \ra) - \sigma(\la v_j, x_{i} \ra))^2\mid\cG_{i-1}] &\geq \frac{1}{n}\sum_{i=1}^n (1/4)\cdot{} e^{- \|\mu_i\|^2} \|u_j-v_j\|^2 \\&\geq (1/4) \cdot{}\|u_j-v_j\|^2 e^{-\frac{1}{n}\sum_{i=1}^n \|\mu_i\|^2},
        \end{align*}
	where the second inequality is simply Jensen's Inequality
        applied to $x\mapsto{}e^{-x}$. This establishes the result for a single pair $\Theta\ind{1}$,
        $\Theta\ind{2}$. To get a statement which holds simultaneously
        for all $\Theta^{(1)}, \Theta^{(2)} \in \Mc$, we apply a union
        bound over a $\eps$--covering set of $\Mc$. In particular, 
        \begin{align*}
        \frac{1}{n}\sum_{i=1}^n \Ex\brk*{\lnorm \sigma(\Theta^{(1)} x_i) - \sigma(\Theta^{(2)} x_i)\rnorm^2\mid\cG_{i-1}} &\geq{} \frac{1}{n}\sum_{i=1}^n \Ex\brk*{\lnorm \sigma(\Theta_{\epsilon}^{(1)} x_i) - \sigma(\Theta_{\epsilon}^{(2)} x_i)\rnorm^2\mid\cG_{i-1}} \\
        &~~~~-\frac{2}{n}\sum_{i=1}^n \Ex\brk*{\lnorm \sigma(\Theta_{\epsilon}^{(1)} x_i) - \sigma(\Theta^{(1)} x_i)\rnorm^2\mid\cG_{i-1}} \\
        &~~~~ -\frac{2}{n}\sum_{i=1}^n \Ex\brk*{\lnorm \sigma(\Theta_{\epsilon}^{(2)} x_i) - \sigma(\Theta^{(2)} x_i)\rnorm^2\mid\cG_{i-1}},
        \end{align*}
        where $\Theta_{\epsilon}^{(i)}$ denotes an arbitrary element
        of the $\epsilon$-net such that $\nrm*{\Theta_{\epsilon}^{(i)}
          - \Theta^{(i)}} \leq \epsilon$. We may take the
        $\eps$-net to have size at most $\left(1+
          \frac{2W}{\eps}\right)^{d^2}$, and so taking a union bound over
        this covering ensures that for any $\Theta^{(1)}, \Theta^{(2)}
        \in \Mc$, whenever
	\begin{equation}
	n \geq \frac{c\tau^4d^3}{1-\rho }\log{\left(1 + \frac{W\trace(K)}{\eps \delta(1-\rho )}\right)}, \label{eq:n_condition} 
	\end{equation}
	we have with probability at least $1-\delta$,
	\begin{align*}
          \frac{1}{n}\sum_{i=1}^n \Ex\left[\lnorm \sigma(\Theta^{(1)} x_i) - \sigma(\Theta^{(2)} x_i)\rnorm^2\mid\cG_{i-1}\right] &\geq{}  (1/4) \cdot{} e^{\frac{-4\rho \trace(K)}{1-\rho }}\|\Theta^{(1)}_{\epsilon}- \Theta^{(2)}_{\epsilon} \|_F^2 -\frac{16\rho \trace\left(K\right)}{1-\rho } \epsilon^2,
	\end{align*}
        where we have used that
        \begin{align}
        \frac{1}{n}\sum_{i=1}^n \Ex\brk*{\lnorm
          \sigma(\Theta_{\epsilon}^{(i)} x_i) - \sigma(\Theta^{(i)}
          x_i)\rnorm^2\mid\cG_{i-1}} \leq \frac{1}{n}\sum_{i=1}^n
          \Ex\brk*{ \nrm*{\Theta_{\epsilon}^{(i)} - \Theta^{(i)}}^2
          \nrm*{x_i}^2 \mid\cG_{i-1}} \leq
          \frac{4\rho \trace\left(K\right)}{1-\rho } \epsilon^2. \label{eq:lipschitz_const}
        \end{align}
Putting everything together, we have
	\begin{align*}
          \frac{1}{n}\sum_{i=1}^n \Ex\left[\lnorm \sigma(\Theta^{(1)} x_i) - \sigma(\Theta^{(2)} x_i)\rnorm^2\mid\cG_{i-1}\right] &\geq{} (1/4) \cdot{} e^{\frac{-4\rho \trace(K)}{1-\rho }}\|\Theta^{(1)}- \Theta^{(2)} \|_F^2 \\
	&~~~~-{} (1/4) \cdot{} e^{\frac{-4\rho \trace(K)}{1-\rho }} \epsilon^2 -\frac{16\rho \trace\left(K\right)}{1-\rho } \epsilon^2.
\end{align*}
	Choosing $\epsilon =
        (1/n)\left(\frac{16\rho \trace\left(K\right)}{1-\rho }\right)^{-1}$
        gives the desired result. The condition on $n$ in
        \pref{eq:n_condition} comes from invoking the requirement that $n \geq \frac{c\tau^4d}{1-\rho }\log{\left(1 +
		\frac{\trace(K)}{\delta(1-\rho )}\right)}$ for a single
            pair with $\delta' \ldef \delta \left(1+ \frac{2W}{\eps}\right)^{-d^2}$.
\end{proof}

\begin{lemma}
	\label{lem:empirical_conditional}
        Define 
	\[
	S \ldef \frac{1}{n}\left|\sum_{i=1}^{n} \lnorm \sigma(\Thetat x_i) - \sigma(\Thetas x_i)\rnorm^2- \Ex\left[\lnorm \sigma(\Thetat x_i) - \sigma(\Thetas x_i)\rnorm^2\mid\cG_{i-1}\right] \right|,
      \]
      and let $R =  c
      \left(\frac{\trace(K)}{(1-\rho )^2}\right)\log{(nd/\delta)} W^2$,
      where $c$ is a sufficiently large numerical constant. Then for
      any $\lambda\in[0,1/R]$, with probability at least $1-\delta$,
	\[
	S \leq  \frac{\lambda R^2}{W^2}  \nrm{\Thetat- \Thetas}_F^2 + \frac{d^2\left(\log{(1/\delta)} + \log{(1+2n\sqrt{R})}\right)}{n\lambda} + \frac{1}{n}.
	\]
\end{lemma}
\begin{proof}
  We first prove that the inequality in the lemma holds with
  $\Theta\ind{t}$ and $\Thetastar$ replaced by any pair $\Theta^{(1)},
  \Theta^{(2)}\in\cM$ fixed a-priori. We then establish the lemma by
  extending this result to a uniform bound for all
  $\Theta\ind{1},\Theta\ind{2}\in\cM$ simultaneously, which in
  particular includes $\Theta\ind{t}$ and $\Thetastar$.

  Let $u_j$ and $v_j$ be the
        $j^{\text{th}}$ rows of $\Theta^{(1)}, \Theta^{(2)}$
        respectively, so that we can write
	\[
	S = \frac{1}{n}\left|\sum_{i=0}^{n-1} \sum_{j=1}^d \left((\sigma(\la u_j, x_{i+1} \ra) - \sigma(\la v_j, x_{i+1} \ra))^2- \Ex[(\sigma(\la u_j, x_{i+1} \ra) - \sigma(\la v_j, x_{i+1} \ra))^2\mid\cG_{i}]\right)\right|.
	\]
	To keep notation compact, for each $i$ define $f_{i} \ldef
        \sum_{j=1}^d (\sigma(\la u_j, x_{i} \ra) - \sigma(\la v_j,
        x_{i} \ra))^2$, so that $S = \frac{1}{n}\left|\sum_{i=1}^n f_i
          - \Ex[f_i \mid \cG_{i-1}]\right|$. Furthermore, define $V
        \ldef \sum_{i=1}^n \Ex[f_i^2 \mid \cG_{i-1}] - (\Ex[f_i \mid
        \cG_{i-1}])^2$. Since $\veps_{i, j}$ is Gaussian, we have
        $\sup_{i, j}|\veps_{i, j}| \leq \sqrt{2 \log{(nd/\delta)}}$
        with probability at least $1-\delta$. Conditioning on the
        event $A\ldef\crl*{\sup_{i, j}|\veps_{i, j}| \leq \sqrt{2
          \log{(nd/\delta)}}}$, we still have 
      \[
      	\Ex\brk*{\veps_{i, j} \mid \cG_{i-1}, A} = 0
      \] 
      so we can apply \pref{lem:freedman_lemma} to $S$. To apply the lemma, we first prove an upper bound on the magnitude of
the iterates $\crl*{x_i}_{i=1}^{n}$. Condition on the event
$A$. Using that $(\Thetas)^{\trn} K \Thetas \preceq \rho  K$, we have 
	\begin{align*}
		\nrm{K^{1/2}x_i} &\leq \nrm{K^{1/2}\sigma(\Thetas x_{i-1})} + \nrm{K^{1/2}\veps_{i-1}} \leq 	 \sqrt{\rho}  \nrm{K^{1/2}x_{i-1}}  + \nrm{K^{1/2} \veps_{i-1}}.
	\end{align*}
	It follows that $\nrm{K^{1/2}x_i} \leq \sum_{j=0}^{i-1} \rho^{\frac{i-1-j}{2}} \cdot{} \nrm{K^{1/2} \veps_{j}}$ and since $\nrm{\veps_{i-1}}^2_K \leq 2\trace(K) \log{(nd/\delta)}$ we have for every $i$, 
	\[
	\nrm{K^{1/2}x_i} \leq \frac{\sqrt{2\trace(K) \log{(nd/\delta)}}}{1 - \sqrt{\rho}} \leq 2 \sqrt{\frac{\trace(K) \log{(nd/\delta)}}{(1 - \rho)^2}},
	\]
	where the second inequality follows from the fact that $\frac{1}{1 -\sqrt{\rho}} = \frac{1 + \sqrt{\rho}}{1 -\rho} \leq \frac{2}{1 -\rho}$. Since $K \succeq I$, this implies
	\[
	\nrm{x_i}^2 \leq c\left(\frac{\trace(K)}{(1-\rho)^2 }\right)\log{(nd/\delta)},
      \]
      where $c$ is a numerical constant. Now, using that $\sigma$ is
      the ReLU and that $\sum_{i=1}^d \nrm{u_j}^2
      \leq W^2$, we have $f_i \leq c
      \left(\frac{\trace(K)}{(1-\rho)^2}\right)\log{(nd/\delta)} W^2 \rdef
      R$. We also have that $V \leq \sum_{i=1}^n \Ex[f_i^2 \mid \cG_{i-1}]
      \leq R \sum_{i=1}^n \Ex[f_i \mid \cG_{i-1}]$. Since $\sigma$ is
      the ReLU, we can upper bound this by
	\begin{align*}
		\sum_{i=1}^n \Ex[f_i \mid \cG_{i-1}] \leq \sum_{i=1}^n \sum_{j=1}^d \Ex[ (\la u_j - v_j, x_{i+1} \ra)^2 \mid \cG_{i-1}] \leq c \frac{\trace(K)}{(1-\rho)^2 }\cdot{}n\log{(nd/\delta)} \sum_{j=1}^d \nrm{u_j - v_j}^2. 
	\end{align*}
	We conclude that $V \leq \frac{nR^2}{W^2} \sum_{j=1}^d
        \nrm{u_j - v_j}^2$. Using \pref{lem:freedman_lemma}, it
        follows that for any $\lambda \in [0, 1/R]$ and any $\delta_0>0$, with probability at least $1-\delta_0$,
	\begin{equation}
	S \leq \lambda \frac{R^2}{W^2} \sum_{j=1}^d \nrm{u_j - v_j}^2 + \frac{\log{(1/\delta_0)}}{n\lambda}. \label{eq:freedman_ub}
	\end{equation}
        Since we conditioned on the event $\sup_{i, j}|\veps_{i, j}|
        \leq \sqrt{2 \log{(nd/\delta)}}$, the final bound
        \pref{eq:freedman_ub} holds with probability at least
        $1-\delta_0-\delta$. We set $\delta_0 = \delta$ to obtain the
        final result.
	
	To get a statement which holds simultaneously for all
        $\Theta^{(1)}, \Theta^{(2)} \in \Mc$, we apply a union bound
        over a $\eps$--covering set for $\Mc$. In particular, there
        exists an $\eps$--covering of $\Mc$ of size at most $\left(1+
          \frac{2W}{\eps}\right)^{d^2}$. Taking a union bound, we have for any $\Theta\ind{1}, \Theta\ind{2} \in \Mc$,
	\begin{equation}
	S \leq \lambda \frac{R^2}{W^2}  \nrm{\Theta^{(1)}- \Theta^{(2)}}_F^2 + \frac{d^2\left(\log{(1/\delta)} + \log{(1+2W/\eps)}\right)}{n\lambda} + \sup_{i} \nrm{x_i}\eps \label{eq:eps_net_2}
      \end{equation}
      The coefficient $\nrm{x_i}$ in front of $\eps$ in
      \pref{eq:eps_net_2} is an upper bound on the Lipschitz constant,
      similar to \pref{eq:lipschitz_const}. Since $\nrm{x_i} \leq \sqrt{R}/W$, setting $\epsilon =
        \frac{W}{n\sqrt{R}}$ leads to the desired result:
	\[
	S \leq \lambda \frac{R^2}{W^2}  \nrm{\Theta^{(1)}- \Theta^{(2)}}_F^2 + \frac{d^2\left(\log{(1/\delta)} + \log{(1+2n\sqrt{R})}\right)}{n\lambda} + \frac{1}{n}.
	\]
\end{proof}

\begin{proof}[\pfref{thm:relu_case}]
	Define $\err = c\tau \sqrt{\frac{d\trace{K}}{n(1-\rho )}\log{\Big(\frac{4\trace(K)}{\delta(1-\rho )}+1\Big)}}$ and $B = \frac{ 4\trace(K)}{(1-\rho )}$. Recall from \pref{eq:first_eq1} and \pref{eq:first_eq2} that  
	\begin{align*}
		\lnorm\Thetat - \Thetas\rnorm^2_F - \lnorm\Thetatt - \Thetas\rnorm^2_F &\geq \frac{2 \lr}{n} \sum_{i=1}^n \la x_{i+1} - \sigma(\Thetat x_i), (\Thetas - \Thetat)x_i \ra \\
		&~~~~- \lr^2 \lnorm \frac{1}{n}\sum_{i=1}^n (x_{i+1} - \sigma(\Thetat x_i))x_i^{\trn} \rnorm_F^2. 
	\end{align*}
Recall from the proof of \pref{thm:known_link_main} that
	\begin{equation*}
		\lr^2 \lnorm \frac{1}{n}\sum_{i=1}^n (x_{i+1} - \sigma(\Thetat x_i))x_i^{\trn} \rnorm_F^2 \leq 2 \lr^2 (\err^2 + \cE({\Theta}^{(t)}) B ) \leq 2 \lr^2 (\err^2 +  \nrm{\Thetat - \Thetas}^2_F\cdot{} B^2  ), 
	\end{equation*}
	since $\cE(\Thetat) \leq  \lnorm\Thetat - \Thetas\rnorm^2_F
        B$, and that
	\begin{align*}
		\frac{2}{n} \sum_{i=1}^n \la x_{i+1} - \sigma(\Thetat x_i), (\Thetas - \Thetat)x_i \ra &= \frac{2}{n} \sum_{i=1}^n \la \sigma(\Thetas x_i) - \sigma(\Thetat x_i), (\Thetas - \Thetat)x_i \ra \\
		&~~~~+ \frac{2}{n} \sum_{i=1}^n \la \varepsilon_{i}, (\Thetas - \Thetat)x_i \ra, \\
		&\geq \frac{2}{n} \sum_{i=1}^n \lnorm\sigma(\Thetas x_i) - \sigma(\Thetat x_i)\rnorm^2 + \frac{2}{n} \sum_{i=1}^n \la \varepsilon_{i}, (\Thetas - \Thetat)x_i \ra.
	\end{align*}
        Let $\delta>0$, be given. Define $R =  \frac{c\trace(K)}{(1-\rho)^2}\log{(nd/\delta)} W^2$
        and let $\lambda \in [0, 1/R]$ be a free parameter.
	Using \pref{lem:empirical_conditional}, we are guaranteed that
        with probability at least $1-\delta$,
	\begin{align*}
		\frac{2}{n} \sum_{i=1}^n \lnorm\sigma(\Thetas x_i) - \sigma(\Thetat x_i)\rnorm^2 &\geq \frac{2}{n} \sum_{i=1}^n \Ex\left[\lnorm \sigma(\Thetat x_i) - \sigma(\Thetas x_i)\rnorm^2\mid\cG_{i-1}\right] \\
		&~~~~- \frac{\lambda R^2}{W^2}  \nrm{\Thetat- \Thetas}_F^2 - \frac{d^2\left(\log{(1/\delta)} + \log{(1+2n\sqrt{R})}\right)}{n\lambda} - \frac{1}{n},
	\end{align*}
For any $\gamma>0$, using the AM-GM inequality, we have
	\begin{align*}
		\frac{2}{n} \sum_{i=1}^n \la \varepsilon_{i}, (\Thetas - \Thetat)x_i \ra &\geq - \lnorm\frac{2}{n} \sum_{i=1}^n \varepsilon_{i} x_i^{\trn} \rnorm_F \lnorm\Thetas - \Thetat\rnorm_F \geq - \gamma \lnorm\frac{2}{n} \sum_{i=1}^n \varepsilon_{i} x_i^{\trn} \rnorm^2_F - \frac{1}{\gamma}\lnorm\Thetas - \Thetat\rnorm_F^2. 
	\end{align*}
It follows that
	\begin{align*}
		\frac{2}{n} \sum_{i=1}^n \la x_{i+1} - \sigma(\Thetat x_i), (\Thetas - \Thetat)x_i \ra &\geq \frac{2}{n} \sum_{i=1}^n \Ex\left[\lnorm \sigma(\Thetat x_i) - \sigma(\Thetas x_i)\rnorm^2\mid\cG_{i-1}\right] \\
		&~~~~- \frac{\lambda R^2}{W^2}  \nrm{\Thetat- \Thetas}_F^2 - \frac{d^2\left(\log{(1/\delta)} + \log{(1+2n\sqrt{R})}\right)}{n\lambda}  \\
		&~~~~- \gamma \lnorm\frac{2}{n} \sum_{i=1}^n \varepsilon_{i} x_i^{\trn} \rnorm^2_F - \frac{1}{\gamma}\lnorm\Thetas - \Thetat\rnorm_F^2 -\frac{1}{n}.
	\end{align*}
	From \pref{lem:relu_lem}, we have that once $n$ is
        sufficiently large (in particular, when the conditions of the
        theorem statement hold), with probability at least $1-\delta$,
	\[
	\frac{1}{n}\sum_{i=1}^n \Ex[(\sigma\Thetat x_i) - \sigma(\Thetas x_i))^2\mid\cG_{i-1}] \geq (1/4) \cdot{} e^{\frac{-4\rho \trace(K)}{1-\rho } }\|u-v\|^2 - 1/n.
	\]
	Choosing $\gamma = e^{\frac{4\rho  \trace(K)}{1-\rho } }$ and $\lambda = \frac{W^2 e^{\frac{-4\rho \trace(K)}{1-\rho }}}{16R^2}$ we then have,  
	\begin{align*}
		\frac{2}{n} \sum_{i=1}^n \la x_{i+1} - \sigma(\Thetat x_i), (\Thetas - \Thetat)x_i \ra &\geq c_1 e^{\frac{-4\rho \trace(K)}{1-\rho } }\|\Thetat-\Thetas\|_F^2 - 4c_2 e^{\frac{4\rho  \trace(K)}{1-\rho }}\lnorm\frac{2}{n} \sum_{i=1}^n \varepsilon_{i} x_i^{\trn} \rnorm^2_F \\
		&~~~~- 4c_3 e^{\frac{4\rho  \trace(K)}{1-\rho }}\frac{R^2 d^2\left(\log{(1/\delta)} + \log{(1+2n\sqrt{R})}\right)}{nW^2} - \frac{2}{n}.
	\end{align*}
	for absolute constants $c_1 \geq \frac{1}{\sqrt{2}}, c_2 > 0,
        c_3 > 0$. Next, using \pref{lem:strong_ub}, we are guaranteed
        that with probability at least $1-\delta$, $\lnorm\frac{2}{n}
        \sum_{i=1}^n \varepsilon_{i} x_i^{\trn} \rnorm^2_F \leq
        \err^2$. Using this along with the choice $\lr =
        \frac{e^{\frac{-4\rho \trace(K)}{1-\rho }}}{16B^2}$, we have
	\begin{align*}
		\lnorm\Thetat - \Thetas\rnorm^2_F - \lnorm\Thetatt - \Thetas\rnorm^2_F &\geq  c_1\cdot{} \lr \cdot{} e^{\frac{-4\rho \trace(K)}{1-\rho }}\|\Thetat-\Thetas\|_F^2 - 4c_2\cdot{} \lr \cdot{} e^{\frac{4\rho  \trace(K)}{1-\rho } }\err^2 \\
		&~~~~- 4c_3 \cdot{} \lr \cdot{} e^{\frac{4\rho  \trace(K)}{1-\rho }}\frac{R^2 d^2\left(\log{(1/\delta)} + \log{(1+2n\sqrt{R})}\right)}{nW^2}\\
		&~~~~-2 \lr^2\cdot{} (\err^2 +  \lnorm\Thetat - \Thetas\rnorm^2_F B^2  ),
	\end{align*}
	which simplifies to
	\begin{align*}
		\lnorm\Thetat - \Thetas\rnorm^2_F - \lnorm\Thetatt - \Thetas\rnorm^2_F &\geq  \frac{c_1^2}{8B^2} \cdot{}  e^{\frac{-8\rho \trace(K)}{1-\rho }}\|\Thetat-\Thetas\|_F^2 \\
		&~~~~- \frac{cR^2 d^2\left(\log{(1/\delta)} + \log{(1+2n\sqrt{R})}\right)}{nB^2W^2}.
	\end{align*}
	Let $\alpha \ldef \frac{c_1^2}{8B^2} \cdot{}
        e^{\frac{-8\rho \trace(K)}{1-\rho }}$. Then the preceding equation simplifies to
	\begin{align*}
		\lnorm\Thetatt - \Thetas\rnorm^2_F  &\leq (1 -\alpha) \cdot \lnorm\Thetat - \Thetas\rnorm^2_F   + \frac{cR^2 d^2\left(\log{(1/\delta)} + \log{(1+2n\sqrt{R})}\right)}{nB^2W^2}.
	\end{align*}
Applying this inequality recursively yields
	\[
	\lnorm\Thetat - \Thetas\rnorm^2_F \leq W^2 (1 - \alpha)^{t} + c \cdot{} e^{\frac{8\rho \trace(K)}{1-\rho } } \cdot{} \frac{R^2 d^2\left(\log{(1/\delta)} + \log{(1+2n\sqrt{R})}\right)}{nW^2}.
	\]
	We conclude that whenever 
	\[
	t \geq 8cB^2 e^{\frac{8\rho \trace(K)}{1-\rho }} \log{\left(\frac{nW^4}{R^2 d^2\left(\log{(1/\delta)} + \log{(1+2n\sqrt{R})}\right)}\right)} \vee 1,\]
	we have
	\[
	\lnorm\Thetat - \Thetas\rnorm^2_F \leq 2c \cdot{} e^{\frac{8\rho \trace(K)}{1-\rho }} \cdot{} \frac{R^2 d^2\left(\log{(1/\delta)} + \log{(1+2n\sqrt{R})}\right)}{nW^2} .
	\]
\end{proof}


\end{document}